\documentclass{article}

\PassOptionsToPackage{round,sort}{natbib}
\usepackage[final]{neurips_2021}

\usepackage[utf8]{inputenc} % allow utf-8 input
\usepackage[T1]{fontenc}    % use 8-bit T1 fonts
\usepackage{url}            % simple URL typesetting

\usepackage{algorithm}
\usepackage{algorithmic}
\usepackage[export]{adjustbox}
\usepackage{subcaption}
\usepackage[percent]{overpic}
\usepackage{wrapfig}
\usepackage{graphicx}
\usepackage{svg}

\usepackage{pifont}% http://ctan.org/pkg/pifont
\usepackage{float}

\usepackage{amsmath,amsfonts,amsthm,amssymb, nccmath}
\usepackage{nicefrac}       % compact symbols for 1/2, etc.
\allowdisplaybreaks % allow page breaks of equations
\usepackage{bbm}
\usepackage{framed}
\usepackage{enumerate}
\usepackage{mathtools}
\usepackage{scalefnt}
\usepackage{changepage}

\usepackage{comment}
\usepackage{setspace}
\usepackage{verbatim}
\usepackage{xcolor}
\usepackage{scalefnt}
\usepackage{xr}

\usepackage[shortlabels]{enumitem}
\setlist{leftmargin=12pt, labelwidth=1pt, topsep=0pt, partopsep=0pt}

\usepackage{placeins}
\usepackage{color}
\usepackage{footmisc}
\usepackage{xspace}

\usepackage{thmtools}
\usepackage{thm-restate}

\usepackage{multirow}
\usepackage{booktabs}
\usepackage{array}
\usepackage{tabularx}

\usepackage{microtype}

\usepackage[colorlinks=true, citecolor=blue, linkcolor=black]{hyperref}

\usepackage{mwe,tikz}
\newcommand{\paragraphsmall}[1]{\textbf{{#1} $\;$}}

\newcommand{\changelocaltocdepth}[1]{%
  \addtocontents{toc}{\protect\setcounter{tocdepth}{#1}}%
  \setcounter{tocdepth}{#1}%
}

\makeatletter
\newcommand{\settitle}{\@maketitle}
\makeatother

\DeclarePairedDelimiterX{\dotp}[2]{\langle}{\rangle}{#1, #2}

\DeclareMathOperator*{\argmin}{\arg\!\min}

\makeatletter
\newcommand*\bigcdot{\mathpalette\bigcdot@{.5}}
\newcommand*\bigcdot@[2]{\mathbin{\vcenter{\hbox{\scalebox{#2}{$\m@th#1\bullet$}}}}}
\makeatother

\newcommand{\be}{\begin{equation}}
\newcommand{\ee}{\end{equation}}

\newcommand{\normsmall}[1]{\| #1\|}              
\newcommand{\norm}[1]{\left\lVert#1\right\rVert}%Alaa modified             

\newcommand{\abs}[1]        {| #1 |}

\newcommand{\eps}{\ensuremath{\varepsilon}}                       % Epsilon
\renewcommand{\epsilon}{\varepsilon}

\DeclarePairedDelimiter{\ceil}{\lceil}{\rceil}

\declaretheorem[name=Theorem]{theorem}

\newcommand{\REAL}{\ensuremath{\mathbb{R}}}

\newcommand{\xmark}{\ding{55}}%

\DeclareMathOperator*{\diag}{diag}

\renewcommand{\paragraph}{\paragraphsmall}

\title{Compressing Neural Networks: Towards Determining the Optimal Layer-wise Decomposition}

\newcommand{\setauthors}[1][$^{*}$]{
\author{%
  Lucas Liebenwein{#1} \\
  MIT CSAIL \\
  \texttt{lucas@csail.mit.edu} \\
  \And
  Alaa Maalouf$^{*}$ \\
  University of Haifa \\
  \texttt{alaamalouf12@gmail.com} \\
  \AND
  Oren Gal \\
  University of Haifa \\
  \texttt{orengal@alumni.technion.ac.il}
  \And
  Dan Feldman \\
  University of Haifa \\
  \texttt{dannyf.post@gmail.com} \\
  \And
  Daniela Rus \\
  MIT CSAIL \\
  \texttt{rus@csail.mit.edu} \\
}
}

\setauthors[\thanks{denotes authors with equal contributions. Code: \url{https://github.com/lucaslie/torchprune}}]

\begin{document}

\maketitle

\changelocaltocdepth{0}
\begin{abstract}
We present a novel global compression framework for deep neural networks that automatically analyzes each layer to identify the optimal per-layer compression ratio, while simultaneously achieving the desired overall compression. Our algorithm hinges on the idea of compressing each convolutional (or fully-connected) layer by slicing its channels into multiple groups and decomposing each group via low-rank decomposition. At the core of our algorithm is the derivation of layer-wise error bounds from the Eckart–Young–Mirsky theorem. We then leverage these bounds to frame the compression problem as an optimization problem where we wish to minimize the maximum compression error across layers and propose an efficient algorithm towards a solution. Our experiments indicate that our method outperforms existing low-rank compression approaches across a wide range of networks and data sets. We believe that our results open up new avenues for future research into the global performance-size trade-offs of modern neural networks.
\end{abstract}
\section{Introduction}\label{sec:intro}

\begin{wrapfigure}[17]{r}{0.44\textwidth}
\vspace{-3mm}
	\centering
	\begin{tikzpicture}
    \node at (-0.1,0) {\includegraphics[width=0.44\textwidth]{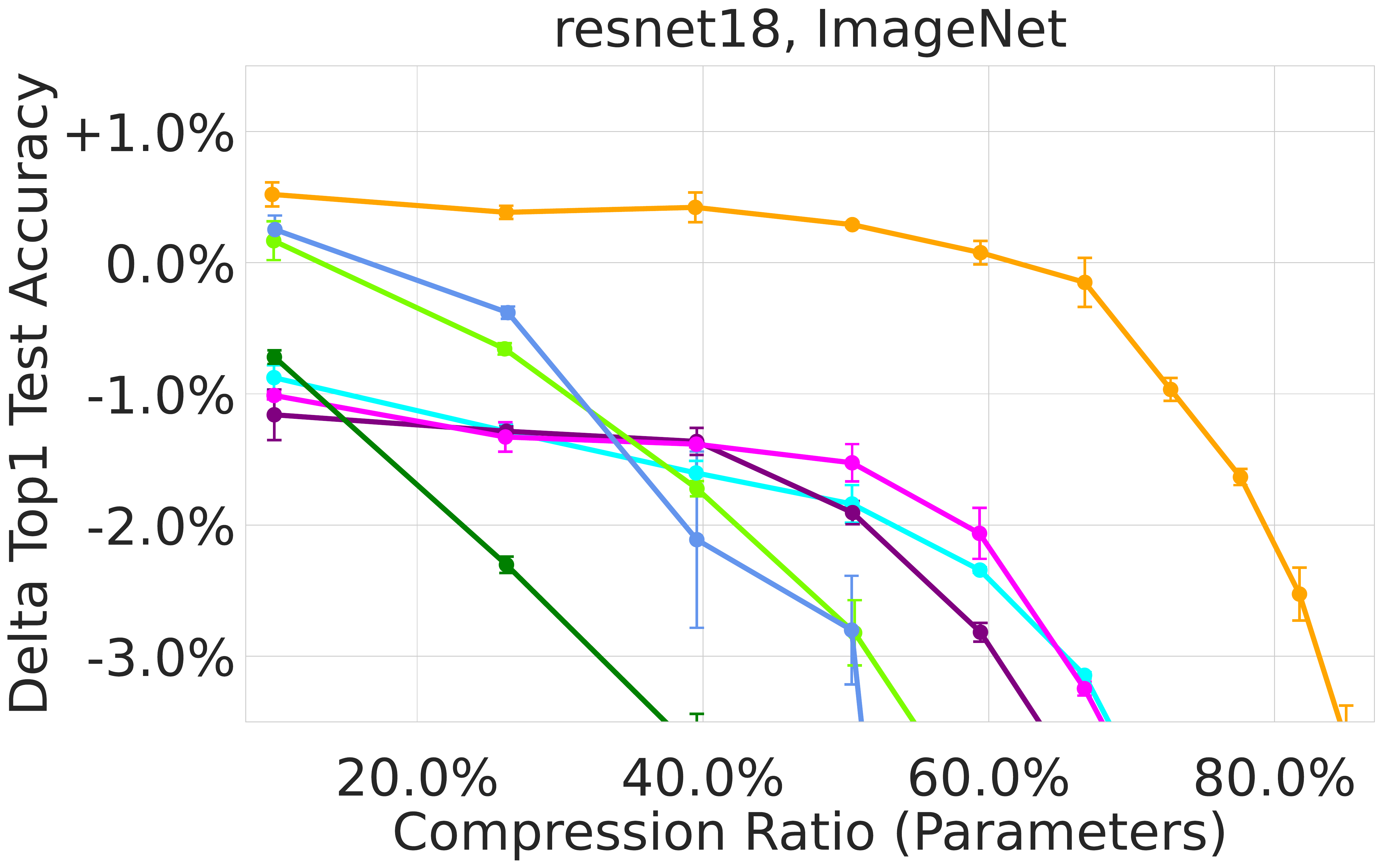}};
    \node at (2.1,0.7) {\includegraphics[width=0.11\textwidth]{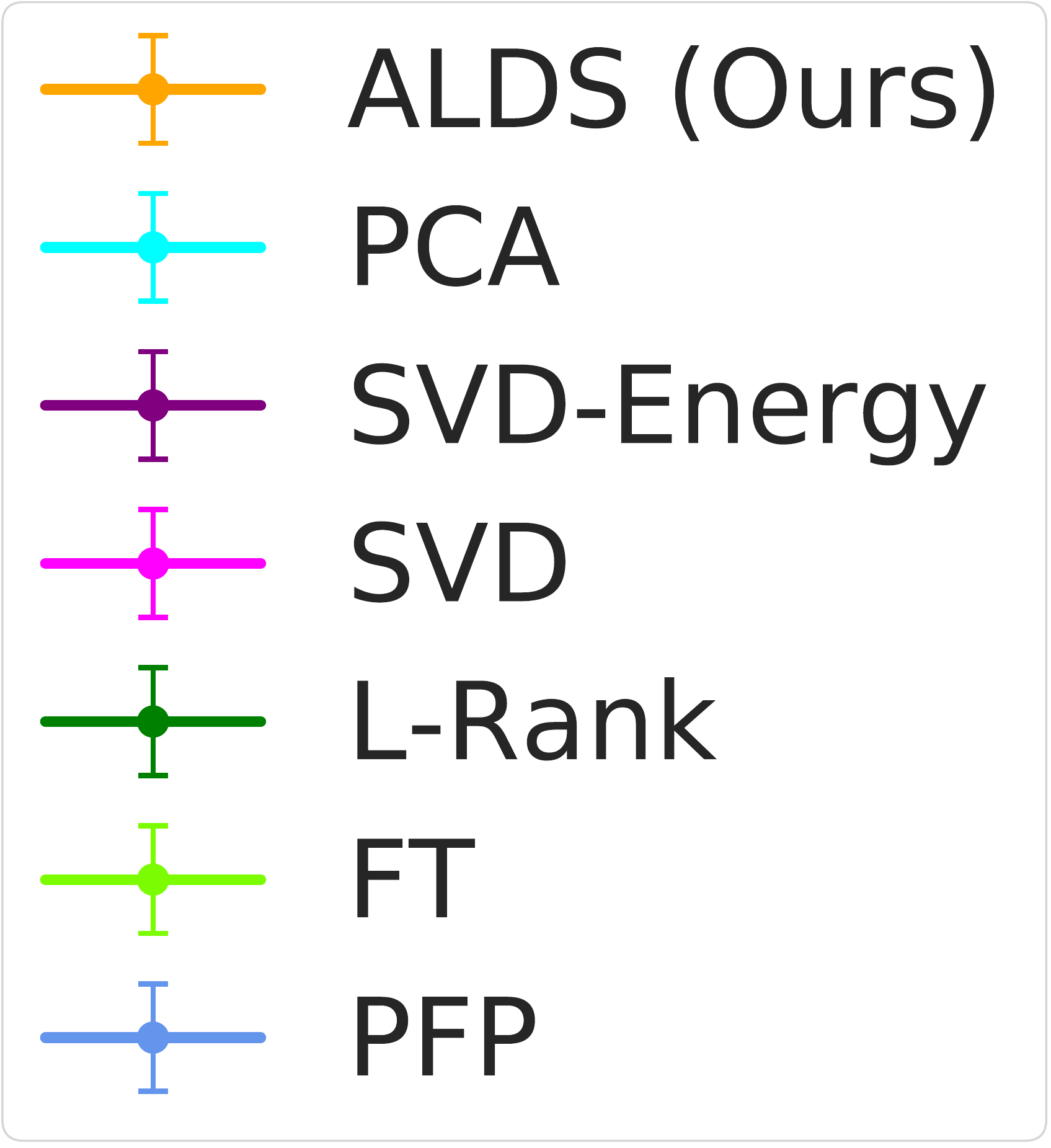}};
    \end{tikzpicture}
	\caption{\textbf{ALDS}, \emph{Automatic Layer-wise Decomposition Selector}, can compress up to 60\% of parameters on a ResNet18 (ImageNet), 3x more compared to baselines. Detailed results are described in Section~\ref{sec:experiments}.}
	\label{fig:intro_results}
\end{wrapfigure}

Neural network compression entails taking an existing model and reducing its computational and memory footprint in order to enable the deployment of large-scale networks in resource-constrained environments. 
Beyond inference time efficiency, compression can yield novel insights into the design~\citep{liu2018rethinking}, training~\citep{liebenwein2021lost, liebenwein2021sparse}, and theoretical properties~\citep{arora2018stronger} of neural networks.

Among existing compression techniques -- which include quantization~\citep{Wu2016}, distillation~\citep{hinton2015distilling}, and pruning~\citep{Han15} -- \emph{low-rank compression} aims at decomposing a layer's weight tensor into a tuple of smaller low-rank tensors. 
Such compression techniques may build upon the rich literature on low-rank decomposition and its numerous applications outside deep learning such as dimensionality reduction~\citep{laparra2015dimensionality} or spectral clustering~\citep{peng2015robust}. Moreover, low-rank compression can be readily implemented in any machine learning framework by replacing the existing layer with a set of smaller layers without the need for, e.g., sparse linear algebra support.

Within deep learning, we encounter two related, yet distinct challenges when applying low-rank compression. On the one hand, each layer should be efficiently decomposed (the ``local step'') and, on the other hand, we need to balance the amount of compression in each layer in order to achieve a desired overall compression ratio with minimal loss in the predictive power of the network (the ``global step''). While the ``local step``, i.e., designing the most efficient layer-wise decomposition method, has traditionally received lots of attention~\citep{Denton2014,kim2015compression,lebedev2014speeding, garipov2016ultimate, novikov2015tensorizing, jaderberg2014speeding}, the ``global step'' has only recently been the focus of attention in research, e.g., see the recent works of~\citet{idelbayev2020low, alvarez2017compression, Xu2020}.

In this paper, we set out to design a framework that simultaneously accounts for both the local and global step. Our proposed solution, termed \emph{Automatic Layer-wise Decomposition Selector} (ALDS), addresses this challenge by iteratively optimizing for each layer's decomposition method (local step) and the low-rank compression itself while accounting for the maximum error incurred across layers (global step). In Figure~\ref{fig:intro_results}, we show how ALDS outperforms existing approaches on the common ResNet18 (ImageNet) benchmark ($60\%$ compression compared to ${\sim}20\%$ for baselines).

\begin{figure}[t!]
\centering
\includegraphics[width=1\textwidth]{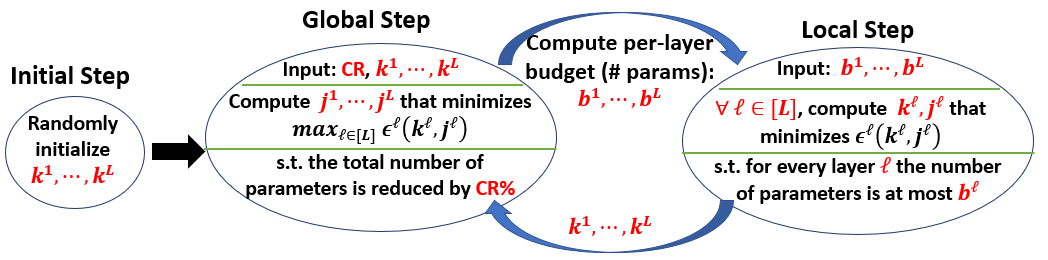}
\caption{\textbf{ALDS Overview}. The framework consists of a global and local step, see  Section~\ref{sec:method}.}
\label{fig:ourpip}
\vspace{-2ex}
\end{figure}

\paragraph{Efficient layer-wise decomposition.}
Our framework relies on a straightforward SVD-based decomposition of each layer. Inspired by~\citet{Denton2014, idelbayev2020low, jaderberg2014speeding} and others, we decompose each layer by first folding the weight tensor into a matrix before applying SVD and encoding the resulting pair of matrices as two separate layers.

\paragraph{Enhanced decomposition via multiple subsets.}
A natural generalization of low-rank decomposition methods entails splitting the matrix into multiple subsets (subspaces) before compressing each subset individually. In the context of deep learning, this was investigated before for individual layers~\citep{Denton2014}, including embedding layers~\citep{maalouf2021deep, Chen2018}. We take this idea further and incorporate it into our layer-wise decomposition method as additional hyperparameter in terms of the number of subsets. Thus, our local step, i.e., the layer-wise decomposition, constitutes of choosing the number of subsets ($k^\ell$) for each layer and the rank ($j^\ell$).

\paragraph{Towards a global solution for low-rank compression.}
We can describe the optimal solution for low-rank compression as the set of hyperparameters (number of subspaces $k^\ell$ and rank $j^\ell$ for each layer in our case) that minimizes the drop in accuracy of the compressed network. While finding the globally optimal solution is NP-complete, we propose ALDS as an efficiently solvable alternative that enables us to search for a locally optimal solution in terms of the maximum relative error incurred across layers. To this end, we derive spectral norm bounds based on the Eckhart-Young-Mirsky Theorem for our layer-wise decomposition method to describe the trade-off between the layer compression and the incurred error. Leveraging our bounds we can then efficiently optimize over the set of possible per-layer decompositions. An overview of ALDS is shown in Figure~\ref{fig:ourpip}.

% \vspace{-2ex}
\newcommand{\Wt}{\mathbf{\mathcal{W}}}
\newcommand{\Wthat}{\hat{\mathbf{\mathcal{W}}}}

\newcommand{\What}{\hat{{{W}}}}
\newcommand{\Xt}{\mathbf{\mathcal{X}}}
\newcommand{\Yt}{\mathbf{\mathcal{Y}}}
\newcommand{\Vt}{\mathbf{\mathcal{V}}}
\newcommand{\Ut}{\mathbf{\mathcal{U}}}

\newcommand{\pr}{\mathbf{CR}}

\newcommand{\W}{W}
\newcommand{\X}{X}
\newcommand{\Y}{Y}
\newcommand{\p}{p}

\section{Method}
\label{sec:method}
In this section, we introduce our compression framework consisting of a layer-wise decomposition method (Section~\ref{sec:method_layer}), a global selection mechanism to simultaneously compress all layers of a network (Section~\ref{sec:method_network}), and an optimization procedure (ALDS) to solve the selection problem (Section~\ref{sec:method_alds}).

\subsection{Local Layer Compression}
\label{sec:method_layer}
We detail our low-rank compression scheme for convolutional layers below and note that it readily applies to fully-connected layers as well as a special case of convolutions with a $1 \times 1$ kernel.

\paragraph{Compressing convolutions via SVD.} 
Given a convolutional layer of $f$ filters, $c$ channels, and a $\kappa_1 \times \kappa_2$ kernel we denote the corresponding weight tensor by $\Wt \in \REAL^{f\times c \times \kappa_1 \times \kappa_2}$. 
Following~\citet{Denton2014, idelbayev2020low, Wen2017} and others, we can then interpret the layer as a linear layer of shape $f \times c \kappa_1 \kappa_2$ and the corresponding rank $j$-approximation as two subsequent linear layers of shape $f \times j$ and $j \times c \kappa_1 \kappa_2$. Mapped back to convolutions, this corresponds to a $j \times c \times \kappa_1 \times \kappa_2$ convolution followed by a $f \times j \times 1 \times 1$ convolution.

\paragraph{Multiple subspaces.}
Following the intuition outlined in Section~\ref{sec:intro} we propose to cluster the columns of the layer's weight matrix into $k \geq 2$ separate subspaces before applying SVD to each subset.
To this end, we may consider any clustering method, such as k-means or projective clustering~\citep{maalouf2021deep,Chen2018}. However, such methods require expensive approximation algorithms which would limit our ability to incorporate them into an optimization-based compression framework as outlined in Section~\ref{sec:method_network}. In addition, arbitrary clustering may require re-shuffling the input tensors which could lead to significant slow-downs during inference.
We instead opted for a simple clustering method, namely \emph{channel slicing}, where we simply divide the $c$ input channels of the layer into $k$ subsets each containing at most $\ceil{c/k}$ consecutive input channels. 
Unlike other methods, channel slicing is efficiently implementable, e.g., as grouped convolutions in PyTorch~\citep{paszke2017automatic} and ensures practical speed-ups subsequent to compressing the network. 

\paragraph{Overview of per-layer decomposition.} 
In summary, for given integers $j,k \geq 1$ and a $4$D tensor $\Wt \in \REAL^{f\times c \times \kappa_1 \times \kappa_2}$ representing a convolution the per-layer compression method proceeds as follows:

\begin{enumerate}
\item 
\textsc{Partition}
the channels of the convolutional layer into $k$ subsets, where each subset has at most $\ceil{c/k}$ consecutive channels, resulting in $k$ convolutional tensors $\{\Wt_{i}\}_{i=1}^k$ where $\Wt_{i}\in \REAL^{f\times c_i \times \kappa_1\times \kappa_2}$, and $\sum_{i=1}^k c_i=c$.% 
\label{step:1}

\item 
\textsc{Decompose}
each tensor $\Wt_{i}$, $i\in [k]$, by building the corresponding weight matrix $\W_i\in \REAL^{f\times c_i \kappa_1 \kappa_2}$, c.f. Figure~\ref{fig:ourarch}, computing its $j$-rank approximation, and factoring it into a pair of smaller matrices $U_{i}$ of $f$ rows and $j$ columns and $V_{i}$ of $j$ rows and $c_i \kappa_1 \kappa_2$ columns.%

\item 
\textsc{Replace}
the original layer in the network by $2$ layers. The first consists of $k$ parallel convolutions, where the $i^{\text{th}}$ parallel layer, $i\in[k]$, is described by the tensor $\Vt_i\in \REAL^{j \times c_i \times \kappa_1 \times \kappa_2}$ which can be constructed from the matrix $V_{i}$ ($j$ filters, $c_i$ channels, $\kappa_1\times \kappa_2$ kernel).
The second layer is constructed by reshaping each matrix $U_i$, $i\in [k]$, to obtain the tensor $\Ut_i \in \REAL^{f \times j \times 1 \times 1}$, and then channel stacking all $k$ tensors $\Ut_1, \cdots ,\Ut_k$ to get a single tensor of shape $f \times kj \times 1 \times 1$.
\end{enumerate}

The decomposed layer is depicted in Figure~\ref{fig:ourarch}.
The resulting layer pair
has $jc\kappa_1\kappa_2$ and $jfk$ parameters, respectively, which implies a parameter reduction from $fc\kappa_1\kappa_2$ to $j(fk+c\kappa_1\kappa_2)$.

\begin{figure}[t]
    %\centering
    \centering
    \includegraphics[width=1\textwidth,height=0.14\textheight]{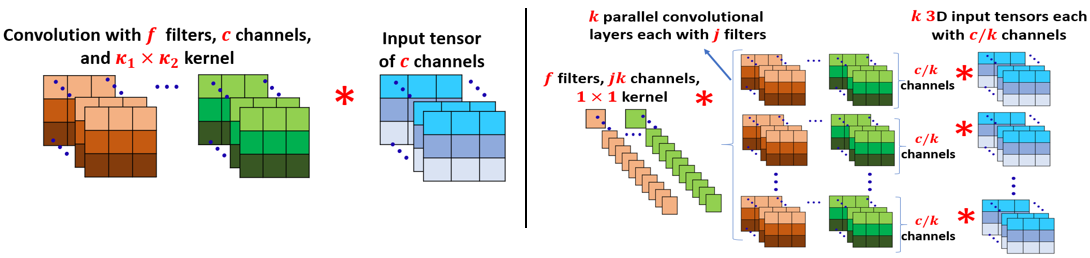}
    \caption{\small \textbf{Left: 2D convolution. right: decomposition used for ALDS.} For a $f \times c \times \kappa_1 \times \kappa_2$ convolution with $f$ filters, $c$ channels, and $\kappa_1 \times \kappa_2$ kernel, our per-layer decomposition consists: (1) $k$ parallel $j \times \nicefrac{c}{k} \times \kappa_1 \times \kappa_2$ convolutions; (2) a single $f \times kj \times 1 \times 1$ convolution applied on the first layer's (stacked) output.}
    \label{fig:ourarch}
\end{figure}

\subsection{Global Network Compression}
\label{sec:method_network}

In the previous section, we introduced our layer compression scheme. We note that in practice we usually want to compress an entire network consisting of $L$ layers up to a pre-specified relative reduction in parameters (``compression ratio'' or $\pr$). However, it is generally unclear how much each layer $\ell \in [L]$ should be compressed in order to achieve the desired $\pr$ while incurring a minimal increase in loss. Unfortunately, this optimization problem is NP-complete as we would have to check every combination of layer compression resulting in the desired $\pr$ in order to optimally compress each layer. On the other hand, simple heuristics, e.g., constant per-layer compression ratios, may lead to sub-optimal results, see Section~\ref{sec:experiments}.
To this end, we propose an efficiently solvable global compression framework based on minimizing the maximum relative error incurred across layers. We describe each component of our optimization procedure in greater detail below.

\paragraph{The layer-wise relative error as proxy for the overall loss.}
Since the true cost (the additional loss incurred after compression) would result in an NP-complete problem, we replace the true cost by a more efficient proxy.
Specifically, we consider the maximum relative error $\eps \coloneqq \max_{\ell \in [L]} \eps^\ell$ across layers, where $\eps^\ell$ denotes the theoretical maximum relative error in the $\ell^\text{th}$ layer as described in Theorem~\ref{thm:rel_error} below.
We choose to minimize this particular cost because: (i) minimizing the maximum relative error ensures that no layer incurs an unreasonably large error that might otherwise get propagated or amplified; (ii) relying on a relative instead of an absolute error notion is preferred as scaling between layers may arbitrarily change, e.g., due to batch normalization, and thus the absolute scale of layer errors may not be indicative of the increase in loss; and (iii) the per-layer relative error has been shown to be intrinsically linked to the theoretical compression error, e.g., see the works of~\citet{arora2018stronger} and~\citet{baykal2018datadependent} thus representing a natural proxy for the cost.

\paragraph{Definition of per-layer relative error.} \label{sec:Per-layerErr}
Let $\Wt^\ell \in \REAL^{f^\ell\times c^\ell \times \kappa^\ell_1 \times \kappa^\ell_2}$ and $\W^\ell\in \REAL^{f^\ell \times c^\ell \kappa_1^\ell \kappa_2^\ell}$ denote the weight tensor and corresponding folded matrix of layer $\ell$, respectively.
The per-layer relative error $\eps^\ell$ is hereby defined as the relative difference in the operator norm between the matrix $\What^\ell$ (that corresponds to the compressed weight tensor $\Wthat^\ell$) and the original weight matrix $\W^\ell$ in layer $\ell$, i.e,. 
\begin{equation}
%\vspace{-1.2ex}
\eps^\ell \coloneqq {\normsmall{\What^\ell - \W^\ell}}/{\normsmall{\W^\ell}}. \label{layer-error}
%\vspace{-1ex}
\end{equation}

Note that while in practice our method decomposes the original layer into a set of separate layers (see Section~\ref{sec:method_layer}), for the purpose of deriving the resulting error we re-compose the compressed layers into the overall matrix operator $\What^\ell$, i.e., $\What^\ell = [U_1^\ell V_1^\ell  \cdots  U_{k^\ell}^\ell V_{k^\ell}^\ell]$, where $U_i^\ell V_i^\ell$ is the factorization of the $i$th cluster (set of columns) in the $\ell$th layer, for every $\ell \in [L]$ and $i \in [k^\ell]$, see supplementary material for more details.
We note that the operator norm $\normsmall{\cdot}$ for a convolutional layer thus signifies the maximum relative error incurred for an individual output patch (``pixel'') across all output channels.

\paragraph{Derivation of relative error bounds.}
We now derive an error bound that enables us to describe the per-layer relative error in terms of the compression hyperparameters $j^\ell$ and $k^\ell$, i.e.,
$
\eps^\ell = \eps^\ell(k^\ell, j^\ell).
$ 
This will prove useful later on as we have to repeatedly query the relative error in our optimization procedure. The error bound is described in the following.

\begin{theorem}
\label{thm:rel_error}
Given a layer matrix $\W^\ell$ and the corresponding low-rank approximation $\What^\ell$, the relative error $\eps^\ell \coloneqq {\normsmall{\What^\ell - \W^\ell}}/{\normsmall{\W^\ell}}$ is bounded by 
\begin{align}
   \eps^\ell \leq \nicefrac{\sqrt{k}}{\alpha_1} \cdot  \max_{i\in [k]}{\alpha_{i,j+1}}, \label{rel-errclaim-main}
\end{align}
where $\alpha_{i,j+1}$ is the $j+1$ largest singular value of the matrix $\W^\ell_i$, for every $i \in [k]$, and $\alpha_1 = \normsmall{\W^\ell}$ is the largest singular value of $\W^\ell$.
\end{theorem}

\begin{proof}
First, we recall the matrices $\W^\ell_1,\cdots, \W^\ell_k$ and we denote the SVD factorization for each of them by: $\W^\ell_i =\tilde{U}^\ell_i \tilde{\Sigma}^\ell_i {{{\tilde{V}}_i}^{\ell}}$. 
Now, observe that for every $i\in [k]$, the matrix $\What^\ell_i$ is the $j$-rank approximation of $\W^\ell_i$. Hence, the SVD factorization of  $\What^\ell_i$ can be writen as $\What^\ell_i = \tilde{U}^\ell_i \hat{\Sigma}^\ell_i {{{\tilde{V}}_i}^{\ell^T}} $, where $\hat{\Sigma}^\ell_i \in \REAL^{f\times d}$ is a diagonal matrix such that its first $j$-diagonal entries are equal to the first $j$-entries on the diagonal of $\tilde{\Sigma}^\ell_i$, and the rest are zeros. Hence,  
\begin{equation}\label{eq:equality-main}
\begin{split}
\W^\ell - \What^\ell  
&= [\W^\ell_1 - \What^\ell_1 ,\cdots, \W^\ell_k -\What^\ell_k]
= [\tilde{U}^\ell_1(\tilde{\Sigma}^\ell_1 - \hat{\Sigma}^\ell_1) \tilde{V}^\ell_1,\cdots, 
\tilde{U}^\ell_k(\tilde{\Sigma}^\ell_k - \hat{\Sigma}^\ell_k) \tilde{V}^\ell_k]
\\&=[\tilde{U^\ell_1} \cdots \tilde{U}^\ell_k]
\diag\left((\tilde{\Sigma}^\ell_1 - \hat{\Sigma}^\ell_1) \tilde{V}^\ell_1, \ldots, (\tilde{\Sigma}^\ell_k - \hat{\Sigma}^\ell_k) \tilde{V}^\ell_k\right).
% \left[\begin{array}{ccc}
% (\tilde{\Sigma}^\ell_1 - \hat{\Sigma}^\ell_1) \tilde{V}^\ell_1 &  & 0\\
% &\ddots&\\
% 0 & & (\tilde{\Sigma}^\ell_k - \hat{\Sigma}^\ell_k) \tilde{V}^\ell_k
% \end{array}\right].
\end{split}
\end{equation}

By~\eqref{eq:equality-main} and by the triangle inequality, we have that

\begin{align}
 \norm{\W^\ell - \What^\ell} \leq  \norm{{\left[\tilde{U^\ell_1} \cdots \tilde{U}^\ell_k\right]}}  
 \norm{\diag\left((\tilde{\Sigma}^\ell_1 - \hat{\Sigma}^\ell_1) \tilde{V}^\ell_1, \ldots, (\tilde{\Sigma}^\ell_k - \hat{\Sigma}^\ell_k) \tilde{V}^\ell_k\right)}.\label{eq:derev-main}
%  \left[\begin{array}{ccc}
% (\tilde{\Sigma}^\ell_1 - \hat{\Sigma}^\ell_1) \tilde{V}^\ell_1 &  & 0\\
% &\ddots&\\
% 0 & & (\tilde{\Sigma}^\ell_k - \hat{\Sigma}^\ell_k) \tilde{V}^\ell_k
% \end{array}\right]
\end{align}

Now, we observe that
\begin{align}
\norm{{\left[\tilde{U^\ell_1} \cdots \tilde{U}^\ell_k\right]}}^2
= \norm{{\left[\tilde{U^\ell_1} \cdots \tilde{U}^\ell_k\right]} 
% \left[\begin{array}{c}
% \tilde{U^\ell_1}\\
% \vdots\\
% \tilde{U^\ell_k}
% \end{array}\right]
{\left[\tilde{U^\ell_1} \cdots \tilde{U}^\ell_k\right]^T}
}
= \norm{
\diag (k, \ldots, k)
% \left[\begin{array}{ccc}
% k &  & 0\\
% & \ddots &\\
% 0 & &k
% \end{array}\right]
} = k. \label{boundingU-main}
\end{align}

Finally, we show that 
\begin{align}
\norm{\diag\left((\tilde{\Sigma}^\ell_1 - \hat{\Sigma}^\ell_1) \tilde{V}^\ell_1, \ldots, (\tilde{\Sigma}^\ell_k - \hat{\Sigma}^\ell_k) \tilde{V}^\ell_k\right)}
&=\max_{i\in [k]}  \norm{(\tilde{\Sigma}^\ell_i - \hat{\Sigma}^\ell_i) \tilde{V}^\ell_i}\\
&=\max_{i\in [k]}  \norm{(\tilde{\Sigma}^\ell_i - \hat{\Sigma}^\ell_i)} = \max_{i\in [k]}\alpha_{i,j+1}, \label{boundingV-main}
\end{align}
where the second equality holds since the columns of $V$ are orthogonal and the last equality holds according to the Eckhart-Young-Mirsky Theorem (Theorem 2.4.8 of~\citet{golub2013matrix}). Plugging~\eqref{boundingV-main} and~\eqref{boundingU-main} into~\eqref{eq:derev-main} concludes the proof.
\end{proof}

\paragraph{Resulting network size.}
Let $\theta = \{\Wt^\ell\}_{\ell=1}^L$ denote the set of weights for the $L$ layers and note that the number of parameters in layer $\ell$ is given by $\abs{\Wt^\ell} = f^\ell c^\ell \kappa^\ell_1 \kappa^\ell_2$ and $|{\theta}| = \sum_{\ell \in [L]} |\Wt^\ell|$. Moreover, note that $|{\Wthat^\ell}| = j^\ell (k^\ell f^\ell + c^\ell \kappa^\ell_1 \kappa^\ell_2)$ if decomposed, $\hat\theta = \{\Wthat^\ell\}_{\ell=1}^L$, and $\abs{\hat\theta} = \sum_{\ell \in [L]} \abs{\Wthat^\ell}$. The overall compression ratio is thus given by
$
1 - \nicefrac{\abs{\hat\theta}}{\abs{\theta}}
$
where we neglected other parameters for ease of exposition. Observe that the layer budget $\abs{\Wthat^\ell}$ is fully determined by $k^\ell, j^\ell$ just like the error bound.

\begin{algorithm}[b!]
\small
\caption{\textsc{ALDS}($\theta$, $\pr$, $n_{\text{seed}}$)}
\label{alg:budget_allocation}
\textbf{Input:} $\theta$: network parameters; $\pr$: overall compression ratio; $n_{\text{seed}}$: number of random seeds to initialize \\
\textbf{Output:} $k^1, \ldots, k^L$: number of subspaces for each layer; $j_1, \ldots, j^L$: desired rank per subspace for each layer
\begin{spacing}{1.1}
\begin{algorithmic}[1]
\small
\FOR{$i \in [n_\text{seed}]$} \label{lin:opt_start}
    \STATE $k^1, \ldots, k^L \gets \textsc{RandomInit()}$ \label{lin:randominit} \\
    \WHILE{not converged} \label{lin:conv_start}
        \STATE $j^1, \ldots, j^L \gets \textsc{OptimalRanks}(\pr, k^1, \ldots, k^L)$ \COMMENT{Global step: choose s.t.\ $\eps^1 = \ldots = \eps^L$} \label{lin:opt_ranks} \\
        \FOR{$\ell \in [L]$}
            \STATE $b^\ell \gets j^\ell (k^\ell f^\ell + c^\ell \kappa_1^\ell \kappa_2^\ell)$ \COMMENT{resulting layer budget} \label{lin:layer_budget} \\
            \STATE $k^\ell \gets \textsc{OptimalSubspaces}(b^\ell)$ \COMMENT{Local step: minimize error bound for a given layer budget} \label{lin:opt_k} \\
        \ENDFOR \label{lin:conv_end}
    \ENDWHILE
    \STATE $\eps_i = \textsc{RecordError}(k^1, \ldots, k^L, j^1, \ldots, j^L)$ \label{lin:res_error} \\
\ENDFOR \label{lin:opt_end}
\STATE \textbf{return} $k^1, \ldots, k^L, j^1, \ldots, j^L$ from $i_{\text{best}} = \argmin_i \eps_i$ \label{lin:best_allocation}
\end{algorithmic}
\end{spacing}
\end{algorithm}

\paragraph{Global Network Compression.}
Putting everything together we obtain the following formulation for the optimal per-layer budget: 
\begin{align}
    \eps_{opt} = 
    & \min_{\substack{\{j^\ell,k^\ell\}_{\ell=1}^L}}  && \max_{\ell \in [L]} \eps^\ell(k^\ell, j^\ell) \label{eq:budget_allocation} \\
    %&\min_{{k^1, \ldots, k^L, j^1, \ldots, j^L}}  && \max_{\ell \in [L]} \eps^\ell(k^\ell, j^\ell) \label{eq:budget_allocation} \\
    & \text{subject to} &&  1 - \nicefrac{\abs{\hat\theta (k^1, j^1, \ldots, k^L, j^L)}}{\abs{\theta}} \leq \pr, \nonumber
\end{align}
where $\pr$ denotes the desired overall compression ratio. Thus optimally allocating a per-layer budget entails finding the optimal number of subspaces $k^\ell$ and ranks $j^\ell$ for each layer constrained by the desired overall compression ratio $\pr$.

\subsection{Automatic Layer-wise Decomposition Selector (ALDS)}
\label{sec:method_alds}

We propose to solve~\eqref{eq:budget_allocation} by iteratively optimizing $k^1, \ldots, k^L$ and $j^1, \ldots, j^L$ until convergence akin of an EM-like algorithm as shown in Algorithm~\ref{alg:budget_allocation} and Figure~\ref{fig:ourpip}. 

Specifically, for a given set of weights $\theta$ and desired compression ratio $\pr$ we first randomly initialize the number of subspaces $k^1, \ldots, k^L$ for each layer (Line~\ref{lin:randominit}). Based on given values for each $k^\ell$ we then solve for the optimal ranks $j^1, \ldots, j^L$ such that the overall compression ratio is satisfied (Line~\ref{lin:opt_ranks}). Note that the maximum error $\eps$ is minimized if all errors are equal. Thus solving for the ranks in Line~\ref{lin:opt_ranks} entails guessing a value for $\eps$, computing the resulting network size, and repeating the process until the desired $\pr$ is satisfied, e.g. via binary search.

Subsequently, we re-assign the number of subspaces $k^\ell$ for each layer by iterating through the finite set of possible values for $k^\ell$ (Line~\ref{lin:opt_k}) and choosing the one that minimizes the relative error for the current layer budget $b^\ell$ (computed in Line~\ref{lin:layer_budget}). Note that we can efficiently approximate the relative error by leveraging Theorem~\ref{thm:rel_error}. We then iteratively repeat both steps until convergence (Lines~\ref{lin:conv_start}-\ref{lin:conv_end}). To improve the quality of the local optimum we initialize the procedure with multiple random seeds (Lines~\ref{lin:opt_start}-\ref{lin:opt_end}) and pick the allocation with the lowest error (Line~\ref{lin:best_allocation}).

We note that we make repeated calls to our decomposition subroutine (i.e. SVD; Lines~\ref{lin:opt_ranks}, \ref{lin:opt_k}) highlighting the necessity for it to be efficient and cheap to evaluate. Moreover, we can further reduce the computational complexity by leveraging Theorem~\ref{thm:rel_error} as mentioned above.

Additional details pertaining to ALDS are provided in the supplementary material.

\paragraph{Extensions.}
Here, we use SVD with multiple subspaces as per-layer compression method. However, we note that ALDS can be readily extended to any desired \emph{set} of low-rank compression techniques. Specifically, we can replace the local step of Line~\ref{lin:opt_k} by a search over different methods, e.g., Tucker decomposition, PCA, or other SVD compression schemes, and return the best method for a given budget. In general, we may combine ALDS with any low-rank compression as long as we can efficiently evaluate the per-layer error of the compression scheme. In the supplementary material, we discuss some preliminary results that highlight the promising performance of such extensions.
\section{Experiments}
\label{sec:experiments}

\begin{figure}[b]
\small
\centering
\begin{minipage}[t]{0.75\textwidth}\vspace{0pt}%
    \centering
    \includegraphics[width=\textwidth]{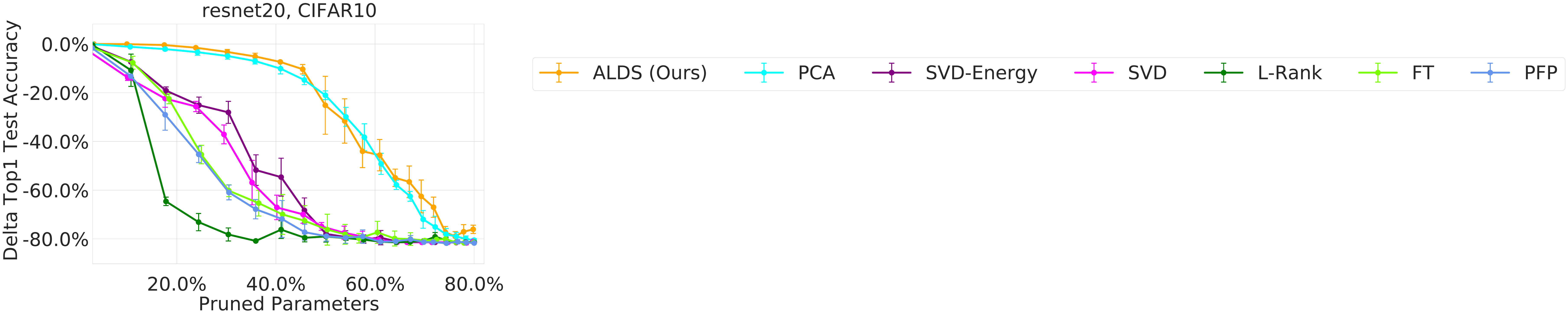}
    \vspace{-3ex}
\end{minipage}%
\hfill
\begin{minipage}[t]{0.32\textwidth}
    \includegraphics[width=\textwidth]{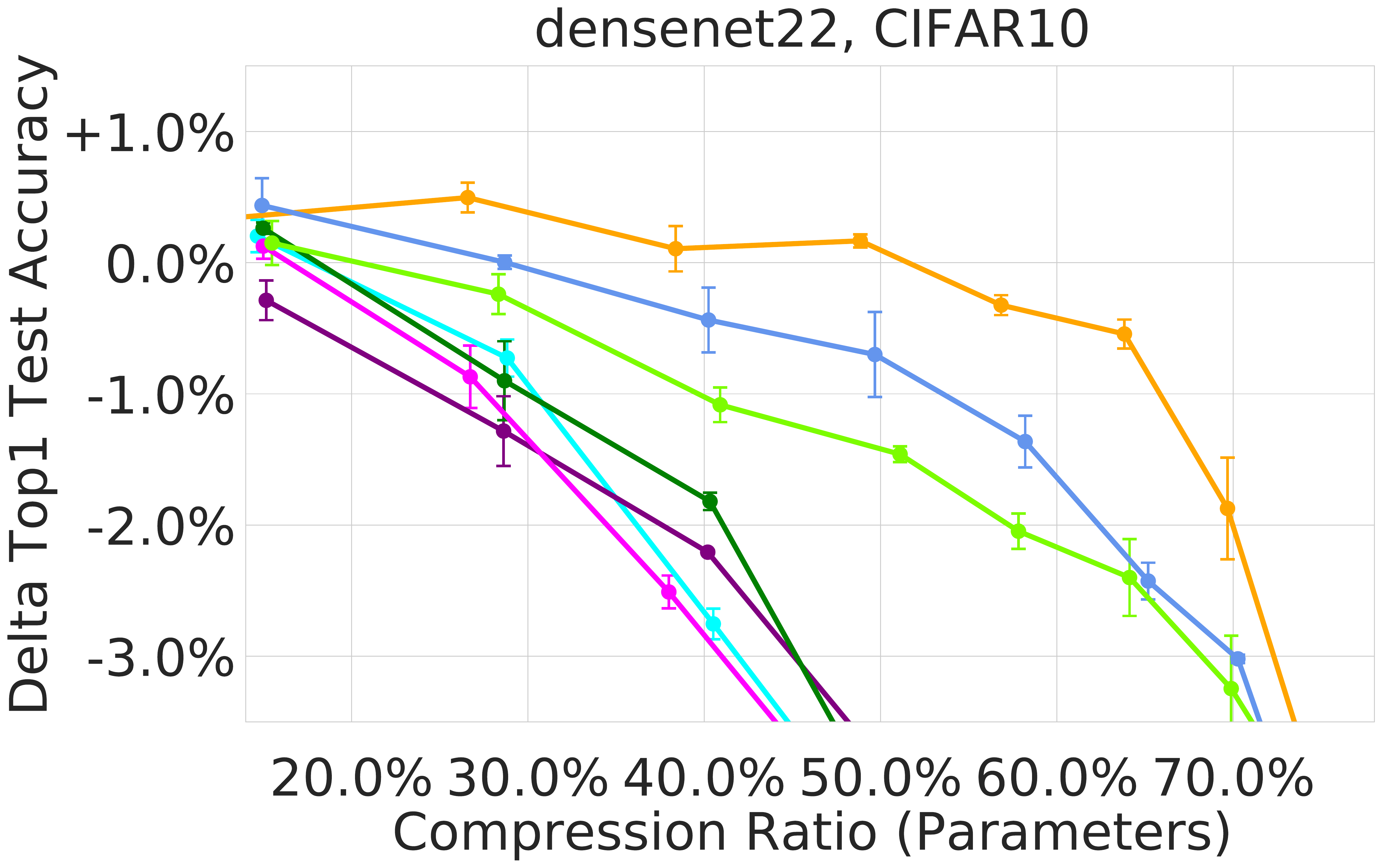}
    \vspace{-3.5ex}
    \subcaption{DenseNet22}
\end{minipage}%
\begin{minipage}[t]{0.32\textwidth}
    \includegraphics[width=\textwidth]{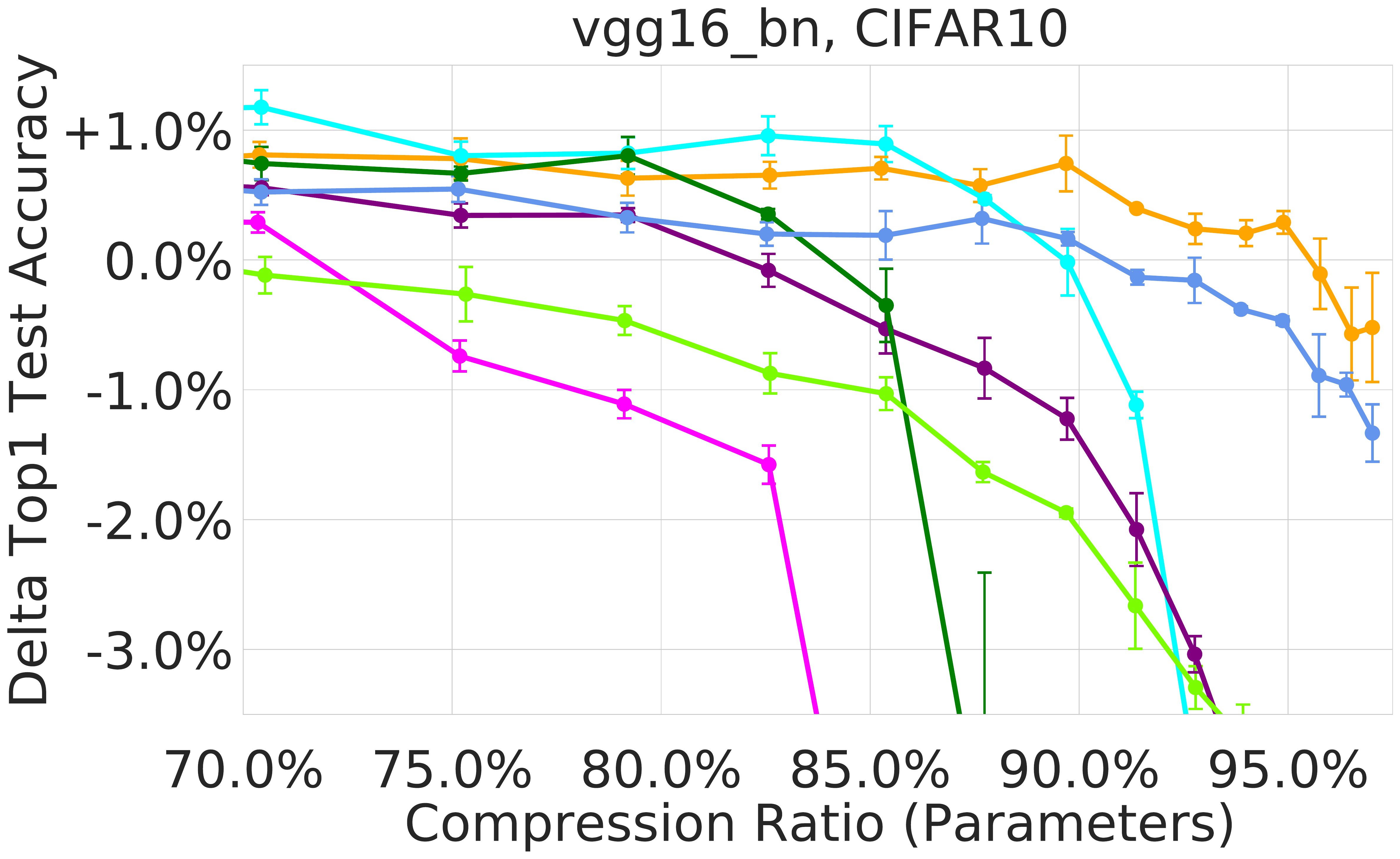}
    \vspace{-3.5ex}
    \subcaption{VGG16}
\end{minipage}%
\begin{minipage}[t]{0.32\textwidth}
    \includegraphics[width=\textwidth]{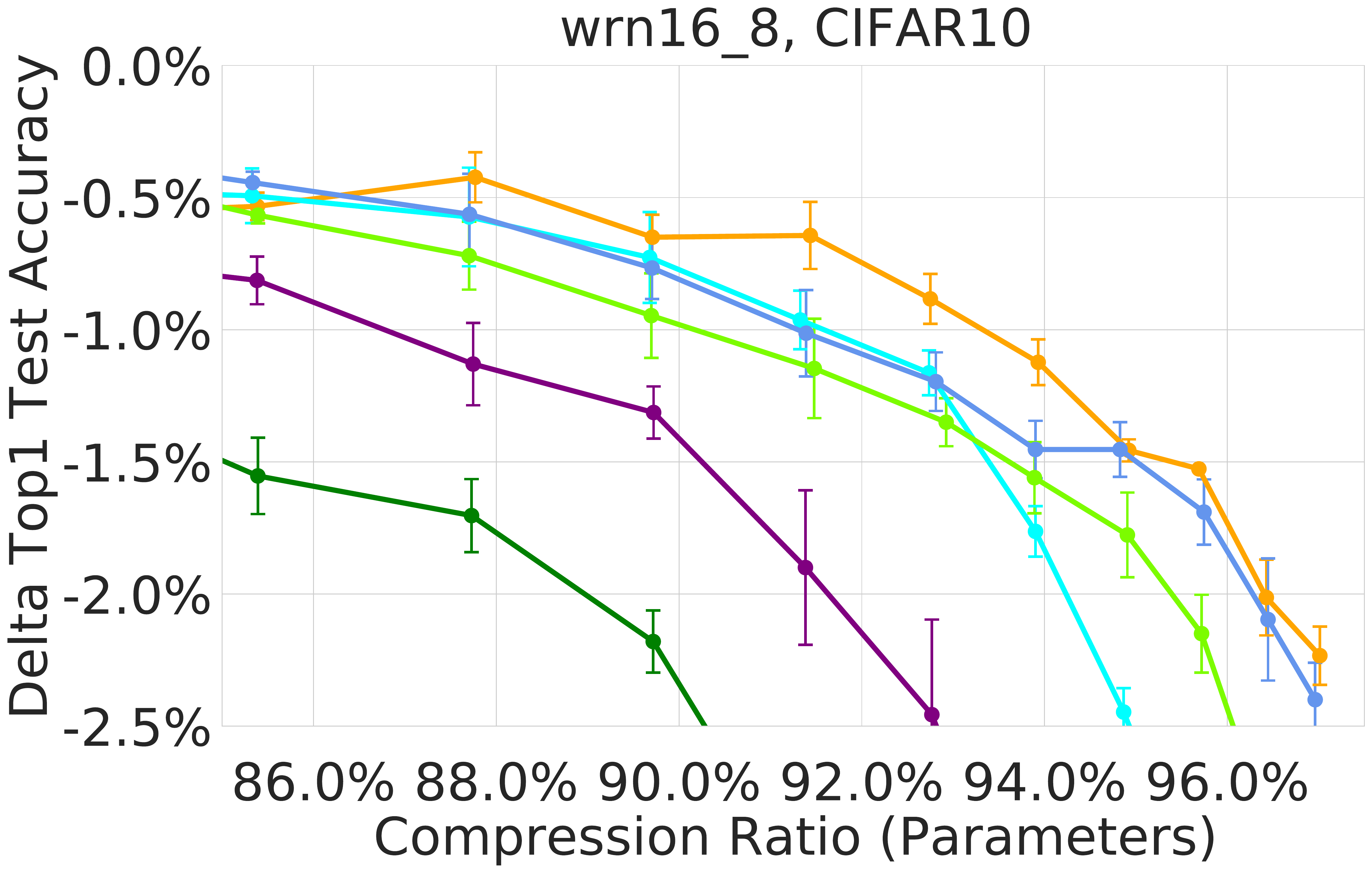}
    \vspace{-3.5ex}
    \subcaption{WRN16-8}
\end{minipage}%
\caption{One-shot compress+retrain experiments on CIFAR10 with baseline comparisons.}
\label{fig:cifar_oneshot}
\end{figure}

\paragraph{Networks and datasets.}
We study various standard network architectures and data sets. Particularly, we test our compression framework on ResNet20~\citep{he2016deep}, DenseNet22~\citep{huang2017densely}, WRN16-8~\citep{zagoruyko2016wide}, and VGG16~\citep{Simonyan14} on CIFAR10~\citep{torralba200880}; ResNet18~\citep{he2016deep}, AlexNet~\citep{Alex2012}, and MobileNetV2~\citep{sandler2018mobilenetv2} on ImageNet~\citep{ILSVRC15}; and on Deeplab-V3~\citep{chen2017rethinking} with a ResNet50 backbone on Pascal VOC segmentation data~\citep{everingham2015pascal}.

\paragraph{Baselines.}
We compare ALDS to a diverse set of low-rank compression techniques. Specifically, we have implemented PCA~\citep{zhang2015accelerating}, SVD with energy-based layer allocation (SVD-Energy) following \citet{alvarez2017compression, Wen2017}, and simple SVD with constant per-layer compression~\citep{Denton2014}. Additionally, we also implemented the recent learned rank selection mechanism (L-Rank) of~\citet{idelbayev2020low}.
Finally, we implemented two recent filter pruning methods, i.e., FT of~\citet{li2016pruning} and PFP of~\citet{liebenwein2020provable}, as alternative compression techniques for densely compressed networks. Additional comparisons on ImageNet are provided in Section~\ref{sec:experiments_iterative}.

\paragraph{Retraining.}
For our experiments, we study one-shot and iterative learning rate rewinding inspired by~\citet{renda2020comparing} for various amounts of retraining. In particular, we consider the following unified compress-retrain pipeline across all methods: 
\begin{enumerate}[leftmargin=12pt, labelwidth=1pt, topsep=0pt, partopsep=0pt,itemsep=0pt]
    \item 
    \textsc{Train} for $e$ epochs according to the standard training schedule for the respective network. 
    \item 
    \textsc{Compress} the network according to the chosen method. 
    \item 
    \textsc{Retrain} the network for $r$ epochs using the training hyperparameters from epochs $[e-r,e]$. 
    \item
    \textsc{Iteratively} repeat  1.-3. after projecting the decomposed layers back (optional).
\end{enumerate}

\begin{figure*}[t!]
\centering
\small
\begin{minipage}[t]{1.0\textwidth}\vspace{0pt}%
    \centering
    \includegraphics[width=0.75\textwidth]{fig/legend/comparisons_horizontal.pdf}
    \vspace{0.0ex}
\end{minipage}
\hfill
\begin{minipage}[t]{0.03\textwidth}%
\vspace{-18ex}%
\rotatebox{90}{ResNet20, CIFAR10}%
\hfill
\end{minipage}%
\begin{minipage}[t]{0.31\textwidth}
    \includegraphics[width=\textwidth]{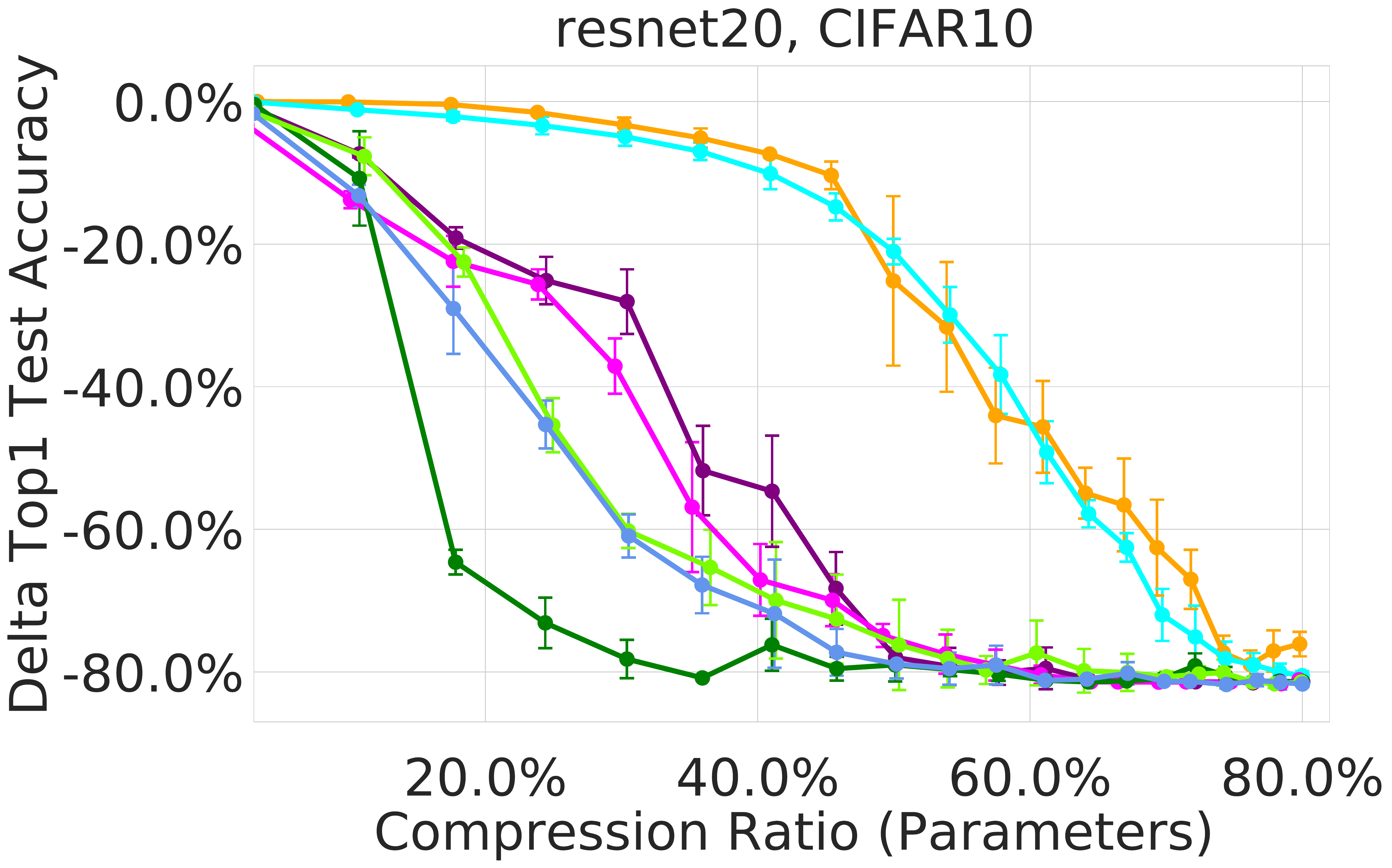}%
    \vspace{-0.5ex}%
    \subcaption{Compress-only (r=0)}
    \label{fig:cifar_resnet20_prune}
\end{minipage}%
\hspace{0.5ex}
\begin{minipage}[t]{0.31\textwidth}
    \includegraphics[width=\textwidth]{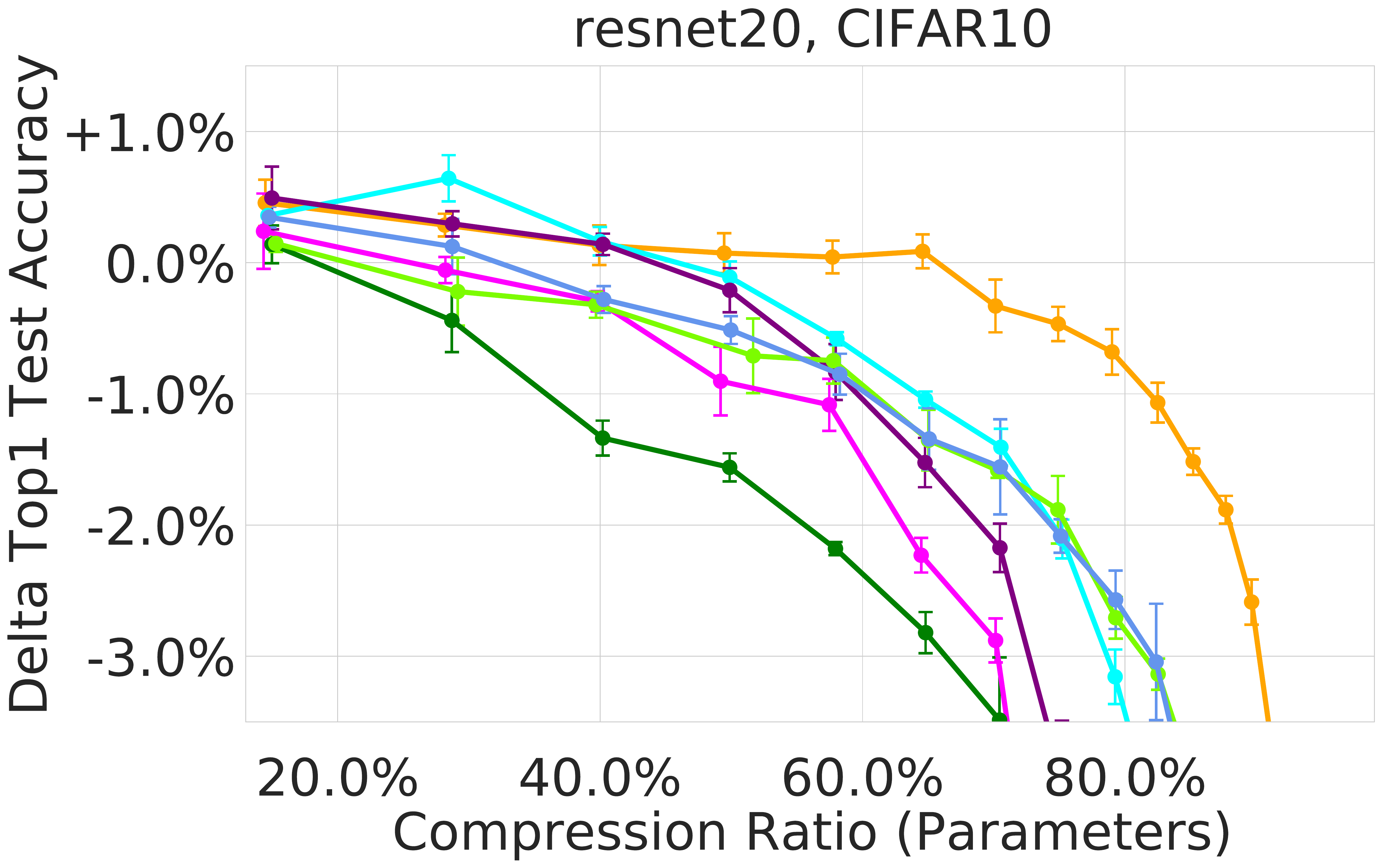}%
    \vspace{-0.5ex}%
    \subcaption{One-shot (r=e)}
    \label{fig:cifar_resnet20_retrain}
\end{minipage}%
\hspace{1ex}
\begin{minipage}[t]{0.31\textwidth}
    \includegraphics[width=\textwidth,height=0.63\textwidth]{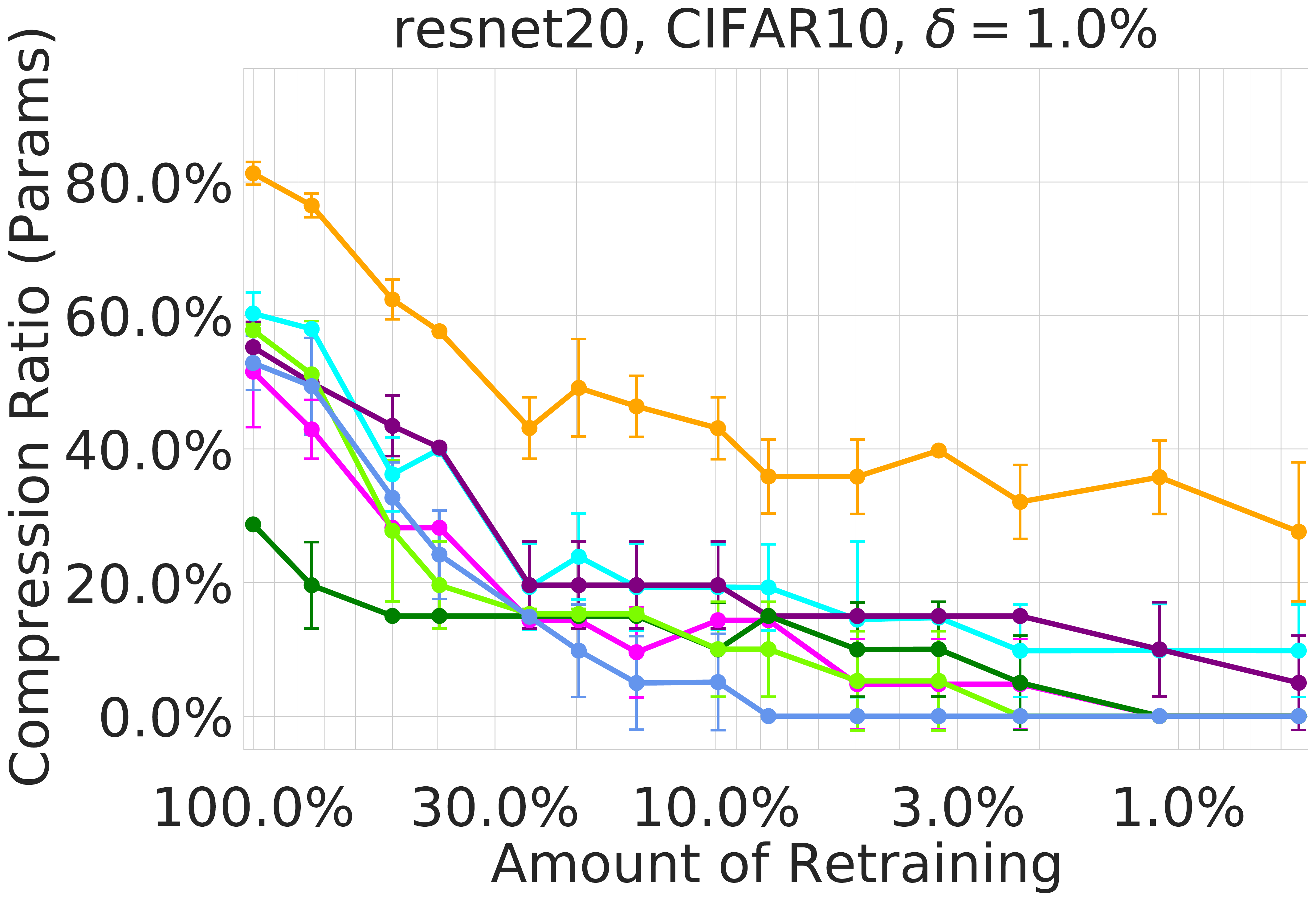}%
    \vspace{-0.5ex}%
    \subcaption{Retrain sweep ($\Delta\text{-Top1}$$\geq$-$1\%$)}
    \label{fig:cifar_resnet20_retrainsweep}
    \vspace{1ex}
\end{minipage}
\begin{minipage}[t]{1.0\textwidth}\vspace{0pt}%
\hspace{5ex}
\hrulefill
\hspace{0.5ex}
\vspace{1ex}
\end{minipage}
\begin{minipage}[t]{0.03\textwidth}%
\vspace{-18ex}%
\rotatebox{90}{ResNet18, ImageNet}%
\hfill
\end{minipage}%
\begin{minipage}[t]{0.31\textwidth}
    \includegraphics[width=\textwidth]{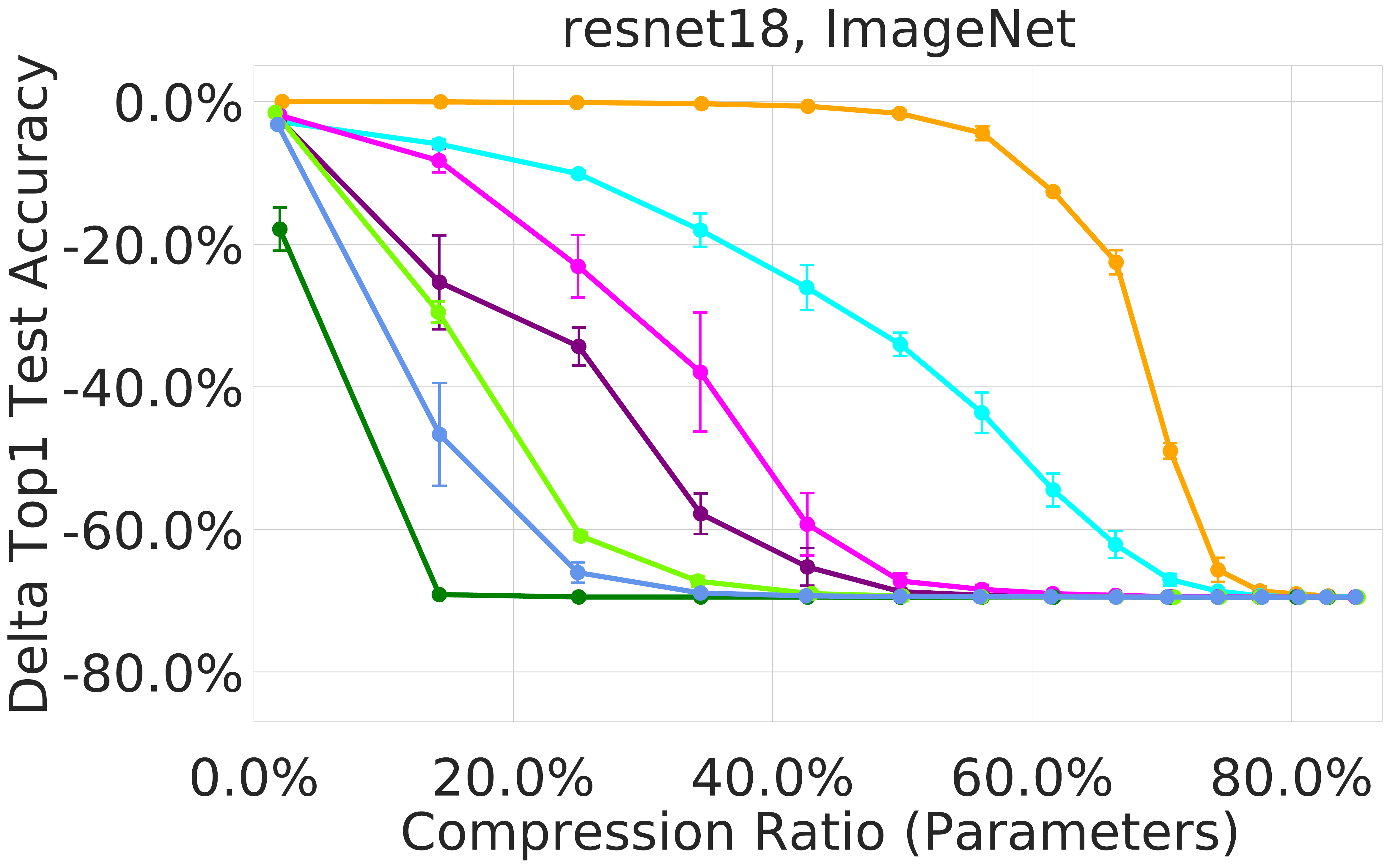}%
    \vspace{-0.5ex}%
    \subcaption{Compress-only (r=0)}
    \label{fig:imagenet_resnet18_prune}
\end{minipage}%
\vspace{1ex}
\begin{minipage}[t]{0.31\textwidth}
    \includegraphics[width=\textwidth]{fig/imagenet_retrain/resnet18_ImageNet_e90_re90_retrain_int10_ImageNet_acc_param.pdf}%
    \vspace{-0.5ex}%
    \subcaption{One-shot (r=e)}
    \label{fig:imagenet_resnet18_retrain}
\end{minipage}%
\hspace{1ex}
\begin{minipage}[t]{0.31\textwidth}
    \includegraphics[width=\textwidth,height=0.63\textwidth]{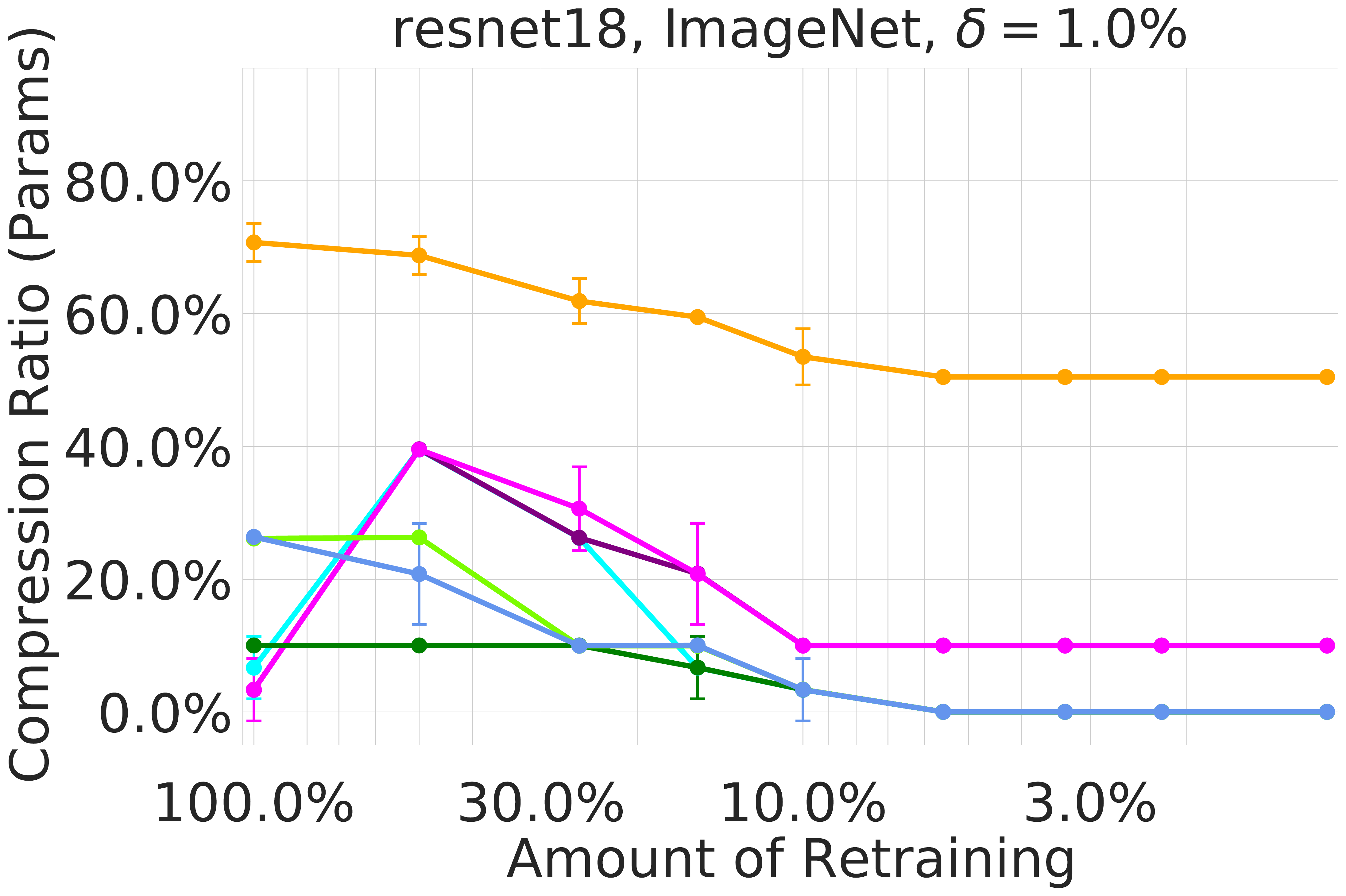}%
    \vspace{-0.5ex}%
    \subcaption{Retrain sweep ($\Delta\text{-Top1}$$\geq$-$1\%$)}
    \label{fig:imagenet_resnet18_retrainsweep}
\end{minipage}
\caption{
The size-accuracy trade-off for various compression ratios, methods, and networks. Compression was performed after training and networks were re-trained once for the indicated amount (\textbf{one-shot}).
(a, b, d, e): the difference in test accuracy for fixed amounts of retraining. (c, f): the  maximal compression ratio with less-than-1\% accuracy drop  for variable amounts of retraining.
}
\label{fig:cifar_imagenet}
\end{figure*}

\paragraph{Reporting metrics.}
We report Top-1, Top-5, and IoU test accuracy as applicable for the respective task. For each compressed network we also report the compression ratio, i.e., relative reduction, in terms of parameters and floating point operations denoted by CR-P and CR-F, respectively.
Each experiment was repeated $3$ times and we report mean and standard deviation.

\begin{figure}[b!]
\small
\centering
\begin{minipage}[t]{0.39\textwidth}
    \includegraphics[width=\textwidth]{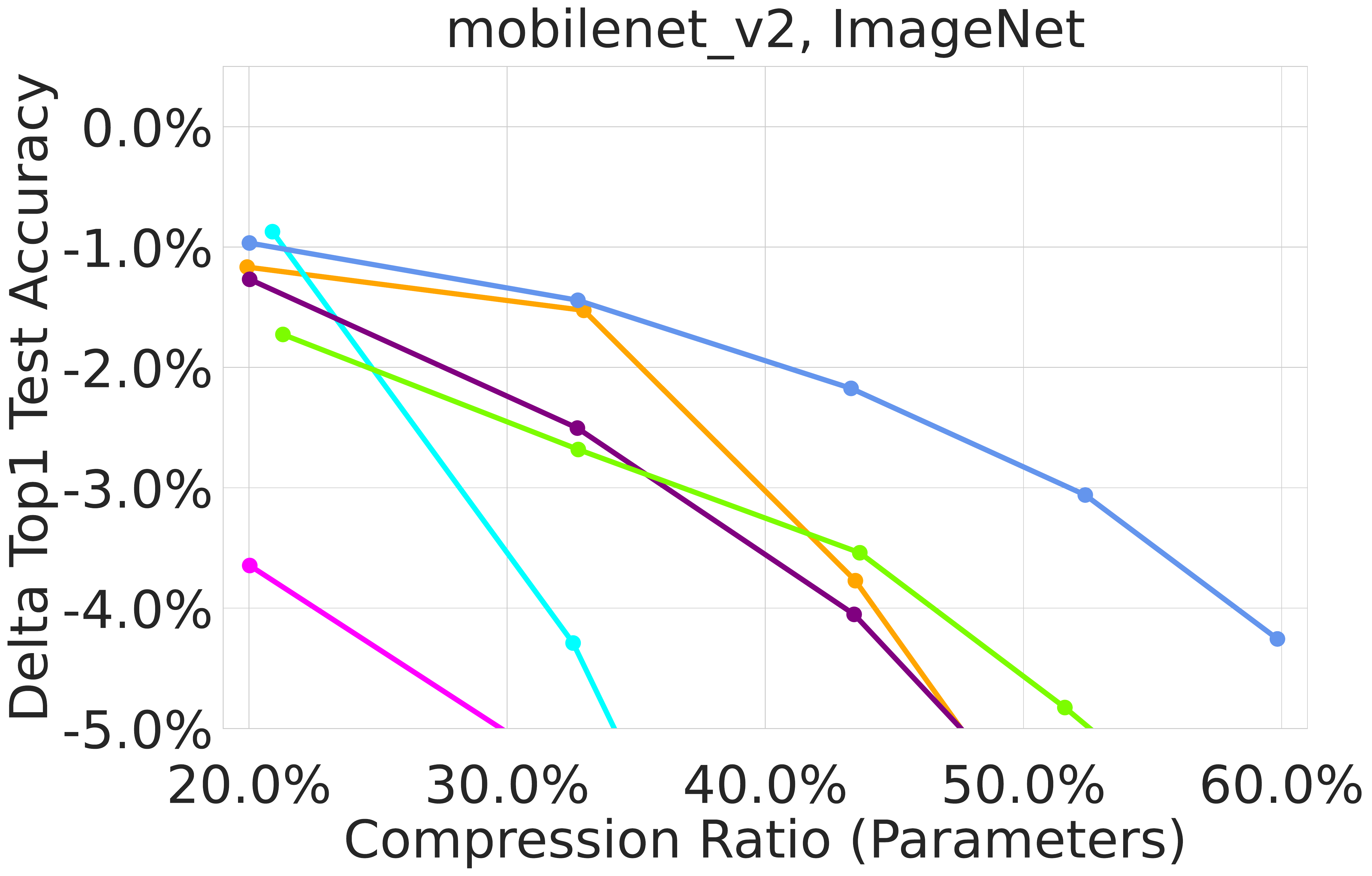}
    \subcaption{MobileNetV2 (ImageNet)}
    \label{fig:mobilenet}
\end{minipage}%
\begin{minipage}[t]{0.39\textwidth}
    \includegraphics[width=\textwidth]{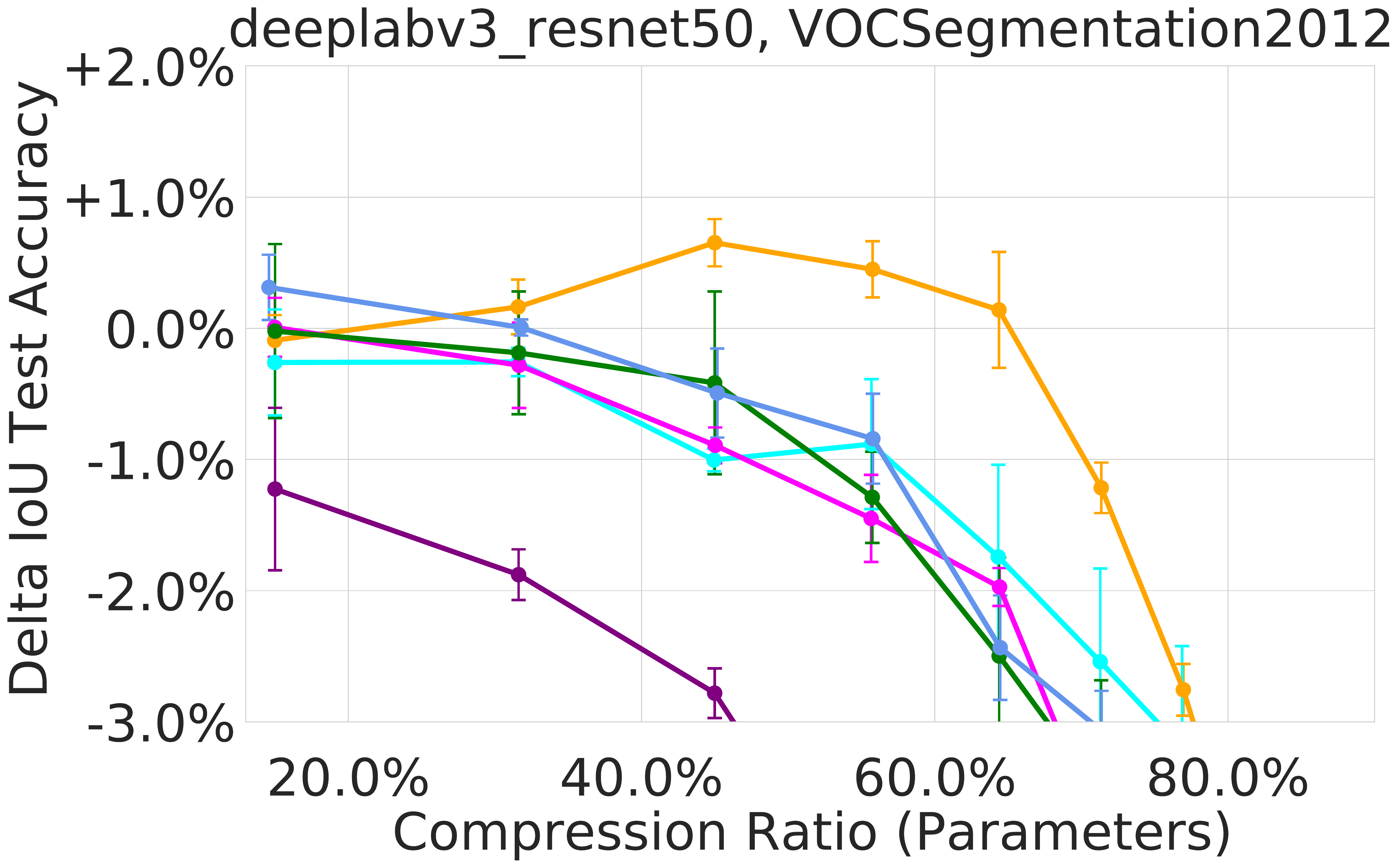}
    \subcaption{DeeplabV3-ResNet50 (VOC)}
    \label{fig:voc}
\end{minipage}%
\begin{minipage}[t]{0.13\textwidth}
    \vspace{-20ex}
    \includegraphics[width=\textwidth]{fig/legend/comparisons.pdf}
    \subcaption*{\phantom{legend}}
\end{minipage}%
\caption{One-shot compress+retrain experiments on various architectures and datasets with baseline comparisons.}
\label{fig:mbv2_deeplab_oneshot}
\end{figure}

% % \newpage
% \begin{wrapfigure}[13]{r}{0.44\textwidth}
% 	\centering
% 	\begin{tikzpicture}
%     \node at (-0.1,0) {\includegraphics[width=0.44\textwidth]{fig/voc_retrain/deeplabv3_resnet50_VOCSegmentation2012_e45_re45_retrain_int10_VOCSegmentation2012_acc_param.pdf}};
%     \node at (2.1,0.75) {\includegraphics[width=0.11\textwidth]{fig/legend/comparisons.pdf}};
%     \end{tikzpicture}
% 	\caption{One-shot compress+retrain for DeeplabV3-ResNet50 on VOC.}
% 	\label{fig:voc}
% \end{wrapfigure}

\subsection{One-shot Compression on CIFAR10, ImageNet, and VOC with Baselines}
\label{sec:experiments_oneshot}
We train reference networks on CIFAR10, ImageNet, and VOC, and then compress and retrain the networks \emph{once} with $r=e$ for various baseline comparisons and compression ratios.

\paragraph{CIFAR10.}
In Figure~\ref{fig:cifar_oneshot}, we provide results for DenseNet22, VGG16, and WRN16-8 on CIFAR10. Notably, our approach is able to outperform existing baselines approaches across a wide range of tested compression ratios. 
Specifically, in the region where the networks incur only minimal drop in accuracy ($\Delta\text{-Top1}$$\geq$$-1\%$) ALDS is particularly effective.

\begin{table}[t!]
\caption{
Baseline results for $\Delta\text{-Top1}$$\geq$$-0.5\%$ for one-shot with highest CR-P and CR-F among tensor decomposition methods bolded for each network. Results coincide with Figures~\ref{fig:cifar_oneshot},~\ref{fig:cifar_imagenet},~\ref{fig:voc}.
}
\label{tab:oneshot}
\begingroup
\setlength{\tabcolsep}{3.05pt} % Default value: 6pt
\renewcommand{\arraystretch}{1.3} % Default value: 1
\centering
\begin{adjustbox}{width=1.0\textwidth}
% scale={1.2}{1.0}
% max height=10ex
\small
\begin{tabular}{|c|c|c||c|cccc|cc|}
\hline
& \multirow{2}{*}{Model}
& \multirow{2}{*}{Metric}
& \multicolumn{5}{c|}{Tensor decomposition}
& \multicolumn{2}{c|}{Filter pruning}
\\
& & 
& \textbf{ALDS (Ours)}
& PCA
& SVD-Energy
& SVD
& L-Rank
& FT
& PFP
\\
\hline \hline
\multirow{8}{*}{\rotatebox{90}{CIFAR10}}
& \multirow{2}{*}{\shortstack{ResNet20 \\ Top1: 91.39}}
& $\Delta$-Top1
& -0.47 & -0.11 & -0.21 & -0.29 & -0.44 & -0.32 & -0.28
\\
& & CR-P, CR-F
& \textbf{74.91}, \textbf{67.86}
& 49.88, 48.67
& 49.88, 49.08
& 39.81, 38.95
& 28.71, 54.89
& 39.69, 39.57
& 40.28, 30.06
\\
\cline{2-10}
& \multirow{2}{*}{\shortstack{VGG16 \\ Top1: 92.78}}
& $\Delta$-Top1
& -0.11 & -0.02 & -0.08 & +0.29 & -0.35 & -0.47 & -0.47 \\
& & CR-P, CR-F
& \textbf{95.77}, \textbf{86.23}
& 89.72, 85.84
& 82.57, 81.32
& 70.35, 70.13
& 85.38, 75.86
& 79.13, 78.44
& 94.87, 84.76
\\
\cline{2-10}
& \multirow{2}{*}{\shortstack{DenseNet22 \\ Top1: 89.88}}
& $\Delta$-Top1
& -0.32 & +0.20 & -0.29 & +0.13 & +0.26 & -0.24 & -0.44 \\
& & CR-P, CR-F
& \textbf{56.84}, \textbf{61.98}
& 14.67, 34.55
& 15.16, 19.34
& 15.00, 15.33
& 14.98, 35.21
& 28.33, 29.50
& 40.24, 43.37
\\
\cline{2-10}
& \multirow{2}{*}{\shortstack{WRN16-8 \\ Top1: 89.88}}
& $\Delta$-Top1
& -0.42 & -0.49 & -0.41 & -0.96 & -0.45 & -0.32 & -0.44 \\
& & CR-P, CR-F
& \textbf{87.77}, 79.90
& 85.33, \textbf{83.45}
& 64.75, 60.94
& 40.20, 39.97
& 49.86, 58.00
& 82.33, 75.97
& 85.33, 80.68
\\
\hline \hline
\multirow{4}{*}{\rotatebox{90}{ImageNet}}
& \multirow{2}{*}{ \shortstack{ResNet18 \\ \scriptsize Top1: 69.62, Top5: 89.08}}
& $\Delta$-Top1, Top5
& -0.40, -0.05 
& -0.95,-0.37 
& -1.49, -0.64
& -1.75, -0.72
& -0.71, -0.23
& +0.10, +0.42
& -0.39, -0.08
\\
& & CR-P, CR-F
& \textbf{66.70}, \textbf{43.51}
& 9.99, 12.78
& 39.56, 40.99
& 50.38, 50.37
& 10.01, 32.64
& 9.86, 11.17
& 26.35, 17.96 
\\
\cline{2-10}
& \multirow{2}{*}{ \shortstack{MobileNetV2 \\ \scriptsize Top1: 71.85, Top5: 90.33}}
& $\Delta$-Top1, Top5
& -1.53, -0.73
& -0.87, -0.55
& -1.27, -0.57
& -3.65, -2.07
& -19.08, -13.40
& -1.73, -0.85
& -0.97, -0.40
\\
& & CR-P, CR-F
& \textbf{32.97}, \textbf{11.01}
& 20.91, 0.26
& 20.02, 8.57
& 20.03, 31.99
& 20.00, 61.97
& 21.31, 20.23
& 20.02, 7.96
\\
\hline \hline
\multirow{1}{*}{\rotatebox{90}{VOC \hspace{0.4ex}}}
& \multirow{2}{*}{ \shortstack{DeeplabV3 \\ \scriptsize IoU: 91.39 Top1: 99.34}}
& $\Delta$-IoU, Top1
& +0.14, -0.15
& -0.26, -0.02
& -1.88, -0.47
& -0.28, -0.18
& -0.42, -0.09
& -4.30, -0.91 
& -0.49, -0.21
\\
& & CR-P, CR-F
& \textbf{64.38}, \textbf{64.11}
& 55.68, 55.82
& 31.61, 32.27
& 31.64, 31.51
& 44.99, 45.02
& 15.00, 15.06
& 45.17, 43.93
\\
\hline
\end{tabular}
\end{adjustbox}
\endgroup
\end{table}

\paragraph{ResNets (CIFAR10 and ImageNet).}
Moreover, we tested ALDS on ResNet20 (CIFAR10) and ResNet18 (ImageNet) as shown in Figure~\ref{fig:cifar_imagenet}. For these experiments, we performed a grid search over both multiple compression ratios and amounts of retraining. Here, we highlight that ALDS outperforms baseline approaches even with significantly less retraining. On Resnet 18 (ImageNet) ALDS can compress over 50\% of the parameters with minimal retraining (1\% retraining) and a less-than-1\% accuracy drop compared to the best comparison methods (40\% compression with 50\% retraining).

\begin{wraptable}{r}{0.5\textwidth}
    \vspace{-1ex}
    \centering
    \caption{AlexNet and ResNet18 Benchmarks on ImageNet. We report Top-1, Top-5 accuracy and percentage reduction of FLOPs (CR-F). Best results with less than 0.5\% accuracy drop are bolded.}
    \label{tab:imagenetbenchmarks}
    \vspace{-1ex}
    \setlength{\tabcolsep}{3.0pt} % Default value: 6pt
    \renewcommand{\arraystretch}{1.2} % Default value: 1
    \small
    \begin{adjustbox}{width=0.5\textwidth}
    \begin{tabular}{|c|l|ccc|}
        \hline
        & Method & $\Delta$-Top1 & $\Delta$-Top5  & CR-F (\%)  \\
        \hline\hline
        \multirow{17}{*}{{\rotatebox{90}{\scriptsize \textbf{ResNet18}, Top1, 5: 69.64\%, 88.98\%}}}
        & ALDS (Ours)
        & \textbf{-0.38} & \textbf{+0.04}  & \textbf{64.5} \\\
        & ALDS (Ours)
        & -1.37 & -0.56  & 76.3 \\
        \cline{2-5}
        & MUSCO~{\tiny\citep{gusak2019automated}}
        & -0.37   &  -0.20  & 58.67\\
        & TRP1~{\tiny\citep{Xu2020}}
        & -4.18     &-2.5     & 44.70\\
        & TRP1+Nu~{\tiny\citep{Xu2020}}
        & -4.25    & -2.61   & 55.15\\
        & TRP2+Nu~{\tiny\citep{Xu2020}}
        & -4.3    & -2.37   & 68.55 \\    
        & PCA~{\tiny\citep{zhang2015accelerating}}
        & -6.54    &-4.54   & 29.07 \\
        & Expand~{\tiny\citep{jaderberg2014speeding}}
        & -6.84    & -5.26  & 50.00 \\
        \cline{2-5}
        & PFP~{\tiny\citep{liebenwein2020provable}}
        &  -2.26  & -1.07 & 29.30 \\
        & SoftNet~{\tiny\citep{he2018soft}}
        & -2.54  & -1.2 &  41.80 \\
        & Median~{\tiny\citep{he2019filter}}
        &  -1.23  & -0.5  & 41.80 \\
        & Slimming~{\tiny\citep{liu2017learning}}
        & -1.77    &   -1.19  & 28.05 \\
        & Low-cost~{\tiny\citep{dong2017more}}
        & -3.55    &  -2.2  &34.64 \\
        & Gating~{\tiny\citep{hua2018channel}}
        & -1.52    &  -0.93 &37.88 \\
        & FT~{\tiny\citep{he2017channel}}
        & -3.08    &  -1.75 & 41.86 \\
        & DCP~{\tiny\citep{zhuang2018discrimination}}
        & -2.19    & -1.28 & 47.08 \\
        & FBS~{\tiny\citep{gao2018dynamic}}
        & -2.44    &  -1.36 & 49.49\\
        \hline\hline
        \multirow{10}{*}{\rotatebox{90}{\scriptsize \textbf{AlexNet}, Top1, 5: 57.30\%, 80.20\%}}
        & ALDS (Ours) 
        & -0.21 & -0.36  & 77.9 \\
        & ALDS (Ours)
        & \textbf{-0.41} & \textbf{-0.54} & \textbf{81.4} \\
        \cline{2-5}
        & Tucker~{\tiny\citep{Kim2015}}
        & N/A  & -1.87 & 62.40 \\
        & Regularize~{\tiny\citep{tai2015convolutional}}
        & N/A  & -0.54 &  74.35\\
        & Coordinate~{\tiny\citep{Wen2017}}
        & N/A  & -0.34 & 62.82 \\
        & Efficient~{\tiny\citep{kim2019efficient}}
        & -0.7 & -0.3  & 62.40 \\
        & L-Rank~{\tiny({\color{blue}Idelbayev et al.,~\citeyear{idelbayev2020low}})}
        & -0.13 & -0.13 & 66.77\\
        \cline{2-5}
        & NISP~{\tiny\citep{yu2018nisp}}
        & -1.43    & N/A  & 67.94 \\ 
        & OICSR~{\tiny\citep{li2019oicsr}}
        & -0.47    &N/A   & 53.70  \\ 
        & Oracle~{\tiny\citep{ding2019approximated}}
        & -1.13    &-0.67 & 31.97 \\
        \hline
    \end{tabular}
    \end{adjustbox}
    \vspace{-12ex}
\end{wraptable}

\paragraph{MobileNetV2 (ImageNet).}
Next, we tested and compared ALDS on the MobileNetV2 architecture for ImageNet as shown in Figure~\ref{fig:mobilenet}. Unlike the other networks, MobileNetV2 is a network already specifically optimized for efficient deployment and includes layer structures such as depth-wise and channel-wise convolutional operations. It is thus more challenging to find redundancies in the architecture. We find that ALDS can outperform existing tensor decomposition methods in this scenario as well.

\paragraph{VOC.}
Finally, we tested the same setup on a DeeplabV3 with a ResNet50 backbone trained on Pascal VOC 2012 segmentation data, see Figure~\ref{fig:voc}. We note that ALDS consistently outperforms other baselines methods in this setting as well (60\% CR-P vs.\ 20\% without accuracy drop).

\paragraph{Tabular results.}
Our one-shot results are again summarized in Table~\ref{tab:oneshot} where we report CR-P and CR-F for $\Delta\text{-Top1}$$\geq$$-0.5\%$. 
We observe that ALDS consistently improves upon prior work. 
We note that pruning usually takes on the order of seconds and minutes for CIFAR and ImageNet, respectively, which is usually faster than even a single training epoch.

\subsection{ImageNet Benchmarks}
\label{sec:experiments_iterative}
% \vspace{-2ex}
Next, we test our framework on two common ImageNet benchmarks, ResNet18 and AlexNet. We follow the compress-retrain pipeline outlined in the beginning of the section and repeat it iteratively to obtain higher compression ratios. Specifically, after retraining and before the next compression step we project the decomposed layers back to the original layer. This way, we avoid recursing on the decomposed layers. 

Our results are reported in Table~\ref{tab:imagenetbenchmarks} where we compare to a wide variety of available compression benchmarks (results were adapted directly from the respective papers). The middle part and bottom part of the table for each network are organized into low-rank compression and filter pruning approaches, respectively. Note that the reported differences in accuracy ($\Delta$-Top1 and $\Delta$-Top5) are relative to our baseline accuracies.
On ResNet18 we can reduce the number of FLOPs by 65\% with minimal drop in accuracy compared to the best competing method (MUSCO, 58.67\%). With a slightly higher drop in accuracy (-1.37\%) we can even compress 76\% of FLOPs. 
On AlexNet, our framework finds networks with -0.21\% and -0.41\% difference in accuracy with over 77\% and 81\% fewer FLOPs. This constitutes a more-than-10\% improvement in terms of FLOPs compared to current state-of-the-art (L-Rank) for similar accuracy drops.

\subsection{Ablation Study}
\label{sec:experiments_ablation}

\begin{wrapfigure}[12]{r}{0.35\textwidth}
    \vspace{-4ex}
	\centering
	\begin{tikzpicture}
    \node at (-0.1,0) {\includegraphics[width=0.35\textwidth]{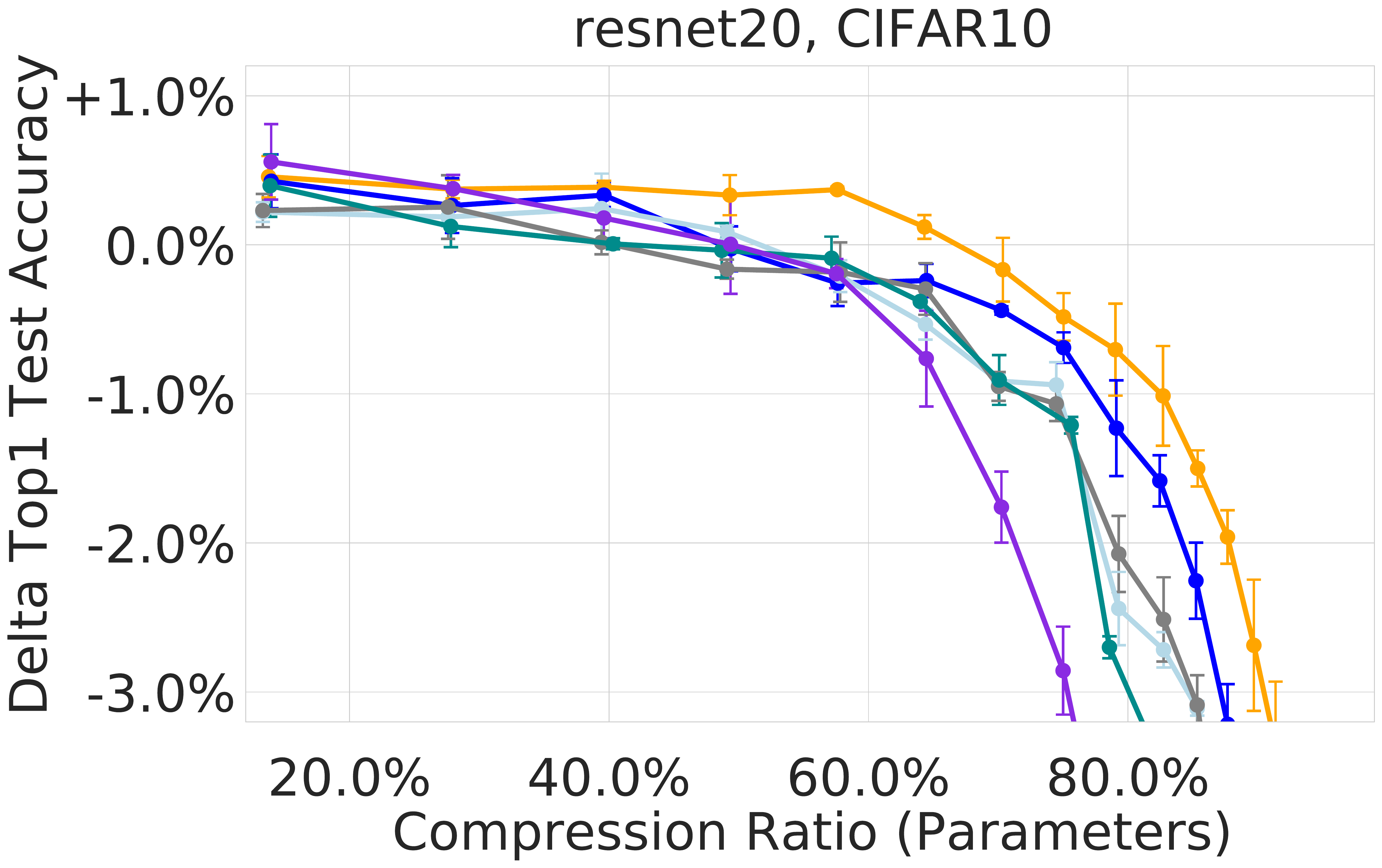}};
    \node at (-0.8,-0.32) {\includegraphics[width=0.11\textwidth]{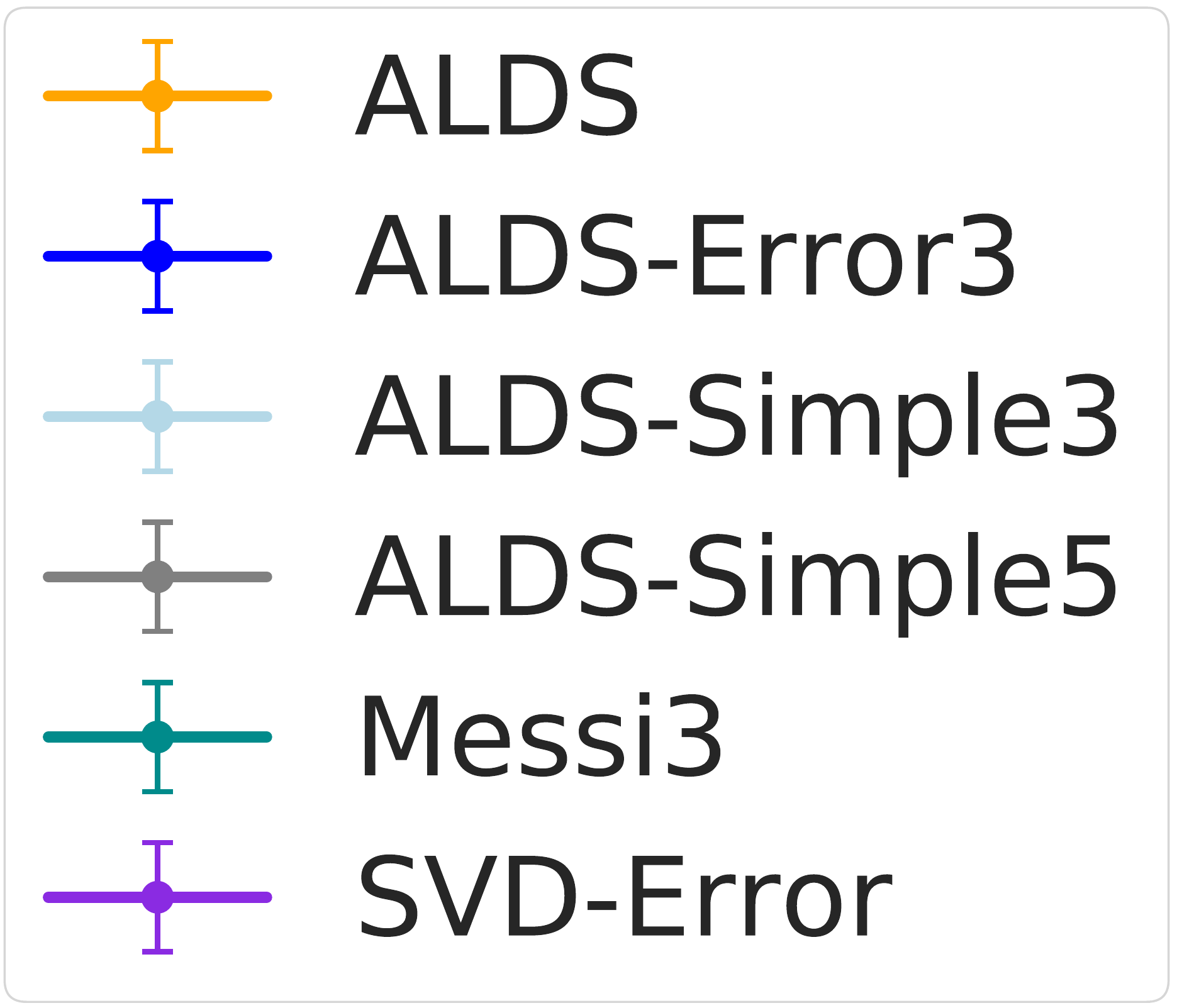}};
    \end{tikzpicture}
	\caption{One-shot ablation study of ALDS for Resnet20 (CIFAR10).}
	\label{fig:cifar_ablation}
\end{wrapfigure}

To investigate the different features of our method we ran compression experiments using multiple variations derived from our method, see Figure~\ref{fig:cifar_ablation}. For the simplest version of our method we consider a constant per-layer compression ratio and fix the value of $k$ to either 3 or 5 for all layers denoted by ALDS-Simple3 and ALDS-Simple5, respectively. Note that ALDS-Simple with $k=1$ corresponds to the SVD comparison method. For the version denoted by ALDS-Error3 we fix the number of subspaces per layer ($k=3$) and only run the global step of ALDS (Line~\ref{lin:opt_ranks} of Algorithm~\ref{alg:budget_allocation}) to determine the optimal per-layer compression ratio. The results of our ablation study in Figure~\ref{fig:cifar_ablation} indicate that our method clearly benefits from the combination of both the global and local step in terms of the number of subspaces ($k$) and the rank per subspace ($j$).

We also compare our subspace clustering (channel slicing) to the clustering technique of~\citet{maalouf2021deep}, which clusters the matrix columns using projective clustering. Specifically, we replace the channel slicing of ALDS-Simple3 with projective clustering (Messi3 in Figure~\ref{fig:cifar_ablation}). 
As expected Messi improves the performance over ALDS-Simple but only slightly and the difference is essentially negligible. Together with the computational disadvantages of Messi-like clustering methods (unstructured, NP-hard; see Section~\ref{sec:method_layer}) ALDS-based simple channel slicing is therefore the preferred choice in our context.

\vspace{-1ex}
\section{Related Work}
\label{sec:related}
\vspace{-1ex}
Our work builds upon prior work in neural network compression. We discuss related work focusing on pruning, low-rank compression, and global aspects of compression.

\paragraph{Unstructured pruning.}
Weight pruning~\citep{singh2020woodfisher,Lin2020Dynamic, molchanov2016pruning, molchanov2019importance, wang2021neural,yu2018nisp} techniques aim to reduce the number of individual weights, e.g., 
%by embedding sparsity as a constraint into the training pipeline~\citep{lebedev2016fast,dong2017learning,aghasi2017net,lin2017runtime}, 
by removing weights with absolute values below a threshold~\citep{Han15, renda2020comparing}, or by  using  a mini-batch of data points to approximate the influence of each parameter on the loss function~\citep{baykal2018datadependent, sipp2019}.
However, since these approaches generate sparse instead of smaller models they require some form of sparse linear algebra support for runtime speed-ups.

\paragraph{Structured pruning.}
Pruning structures such as filters directly shrinks the network~\citep{li2019learning,luo2020autopruner,liu2019metapruning,chen2020storage, ye2018rethinking,lin2020hrank}.
Filters can be pruned using a score for each filter, e.g., weight-based~\citep{he2017channel, he2018soft} or data-informed~\citep{yu2018nisp, liebenwein2020provable}, and removing those with a score below a threshold.
It is worth noting that filter pruning is complimentary to low-rank compression.

\paragraph{Low-rank compression (local step).} 
A common approach to low-rank compression entails tensor decomposition including Tucker-decomposition~\citep{kim2015compression}, CP-decomposition~\citep{lebedev2014speeding}, Tensor-Train~\citep{garipov2016ultimate,novikov2015tensorizing} and others~\citep{Denil2013,jaderberg2014speeding,ioannou2017deep}. Other decomposition-like approaches include weight sharing, random projections, and feature hashing~\citep{Weinberger09,arora2018stronger, shi2009hash, Chen15Hash, Chen15Fresh, ullrich2017soft}.
Alternatively, low-rank compression can be performed via matrix decomposition (e.g., SVD) on flattened tensors as done by~\citet{Denton2014,yu2017compressing,sainath2013low,xue2013restructuring, tukan2020compressed} among others. 
\citet{Chen2018, maalouf2021deep,Denton2014} also explores the use of subspace clustering before applying low-rank compression to each cluster to improve the approximation error.
Notably, most prior work relies on some form of expensive approximation algorithm -- even to just solve the per-layer low-rank compression, e.g., clustering or tensor decomposition. In this paper, we instead focus on the global compression problem and show that simple compression techniques (SVD with channel slicing) are advantageous in this context as we can use them as efficient subroutines. We note that we can even extend our algorithm to multiple, different types of per-layer decomposition.

\paragraph{Network-aware compression (global step).} To determine the rank (or the compression ratio) of each layer, prior work suggests to account for compression during training~\citep{alvarez2017compression,Wen2017,Xu2020,ioannou2015training,ioannou2016training}, e.g, by training the network with a penalty that encourages the weight matrices to be low-rank. 
Others suggest to select the ranks using variational Bayesian matrix factorization~\citep{kim2015compression}. 
In their recent paper,~\citet{chin2020towards} suggest to produce an entire set of compressed networks with different accuracy/speed trade-offs. 
Our paper was also inspired by a recent line of work towards automatically choosing or learning the rank of each layer~\citep{idelbayev2020low,gusak2019automated,tiwari2021chipnet,zhang2015efficient,li2018constrained,zhang2015accelerating}.
We take such approaches further and suggest a global compression framework that incorporates multiple decomposition techniques with more than one hyper-parameter per layer (number of subspaces and ranks of each layer). This approach increases the number of local minima in theory and helps improving the performance in practice.
\vspace{-2ex}
\section{Discussion and Conclusion}
\vspace{-2ex}
\paragraph{Practical benefits.}
By conducting a wide variety of experiments across multiple data sets and networks we have shown the effectiveness and versatility of our compression framework compared to existing methods. The runtime of  ALDS is negligible compared to retraining and it can thus be efficiently incorporated into compress-retrain pipelines.

\paragraph{ALDS as modular compression framework.}
By separately considering the low-rank compression scheme for each layer (local step) and the actual low-rank compression (global step) we have provided a framework that can efficiently search over a set of desired hyperparameters that describe the low-rank compression. Naturally, our framework can thus be generalized to other compression schemes (such as tensor decomposition) and we hope to explore these aspects in future work.

\paragraph{Error bounds lead to global insights.}
At the core of our contribution is our error analysis that enables us to link the global and local aspects of layer-wise compression techniques. We leverage our error bounds in practice to compress networks more effectively via an automated rank selection procedure without additional tedious hyperparameter tuning. 
However, we also have to rely on a proxy definition (maximum relative error) of the compression error to enable a tractable solution that we can implement efficiently.
We hope these observations invigorate future research into compression techniques that come with tight error bounds -- potentially even considering retraining -- which can then naturally be wrapped into a global compression framework.

\label{sec:conclusion}

\section*{Acknowledgments}
This research was sponsored by the United States Air Force Research Laboratory and the United States Air Force Artificial Intelligence Accelerator and was accomplished under Cooperative Agreement Number FA8750-19-2-1000. The views and conclusions contained in this document are those of the authors and should not be interpreted as representing the official policies, either expressed or implied, of the United States Air Force or the U.S. Government. The U.S. Government is authorized to reproduce and distribute reprints for Government purposes notwithstanding any copyright notation herein.
This work was further supported by the Office of Naval Research (ONR) Grant N00014-18-1-2830.

% \section*{Funding Transparency Statement}
% Authors declare no competing interests.
% \textit{Funding in direct support of this work:} the United States Air Force Research Laboratory and the United States Air Force Artificial Intelligence Accelerator accomplished under Cooperative Agreement Number FA8750-19-2-1000 and the Office of Naval Research (ONR) Grant N00014-18-1-2830.

\bibliographystyle{plainnat}

\bibliography{misc/references}

\appendix
\changelocaltocdepth{2}
\clearpage
% \title{Supplementary Material:\\``Compressing Neural Networks: Towards Determining the Optimal Layer-wise Decomposition''}
\title{Supplementary Material}
\setauthors
% \maketitleextra
\maketitleextranoauthors
\tableofcontents
\newpage

\section{Further Method Details}
\label{sec-supp:method}
In this section, we provide more details pertaining to our method.

\subsection{Method Preliminaries}

Our layer-wise compression technique hinges upon the insight that any linear layer may be cast as a matrix multiplication, which enables us to rely on SVD as compression subroutine. Focusing on convolutions we show how such a layer can be recast as matrix multiplication. Similar approaches have been used by~\citet{Denton2014, idelbayev2020low, Wen2017} among others.

\begin{figure}
    \centering
    \includegraphics[width=0.75\textwidth]{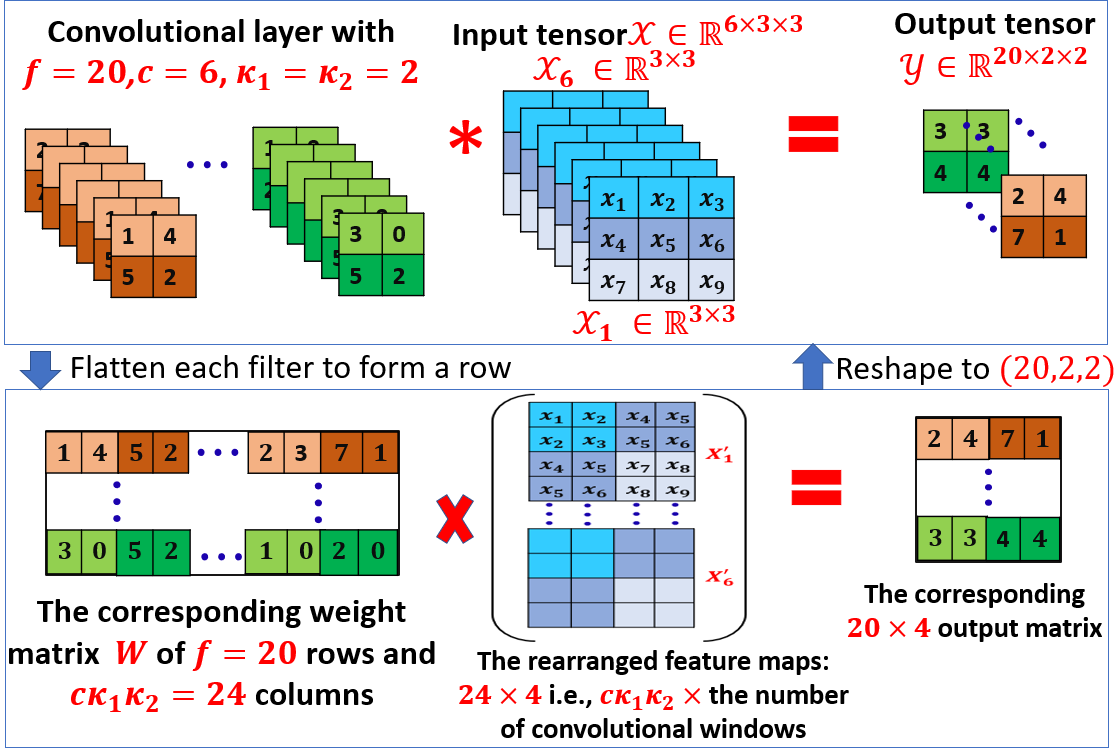}
    \caption{\small \textbf{Convolution to matrix multiplication.} A convolutional layer of $f=20$ filters, $c=6$ channels, and $2\times2$ kernel ($\kappa_1=\kappa_2=2$). The input tensor shape is $6 \times 3\times 3$. The corresponding weight matrix has $f=20$ rows (one row per filter) and $24$ columns ($c \times \kappa_1 \times \kappa_2$), as for the corresponding feature matrix, it has $24$ rows and $4$ columns, the $4$ here is the number of convolution windows (i.e., number of pixels/entries in each of the output feature maps). After multiplying those matrices, we reshape them to the desired shape to obtain the desired output feature maps.}
    \label{fig:conv-matmul}
\end{figure}

\paragraph{Convolution to matrix multiplication.}
For a given convolutional layer of $f$ filters, $c$ channels, $\kappa_1 \times \kappa_2$ kernel and an input feature map with $c$ features, each of size $m_1\times m_2$, we denote by $\Wt \in \REAL^{f\times c \times \kappa_1 \times \kappa_2}$ and $\Xt \in \REAL^{c \times m_1 \times m_2}$ the weight tensor and input tensor, respectively. 
Moreover, let $\W\in \REAL^{f \times c\kappa_1\kappa_2}$ denote the unfolded matrix operator of the layer constructed from $\Wt$ by flattening the $c$ kernels of each filter into a row and stacking the rows to form a matrix.
Finally, let $p$ denote the total number of sliding blocks and $\X \in \REAL^{c\kappa_1\kappa_2 \times \p}$ denote the unfolded input matrix, which is constructed from the input tensor $\Xt$ as follows: while simulating the convolution by sliding $\Wt$ along $\Xt$ we extract the sliding local blocks of $\Xt$ across all channels by flattening each block into a $c\kappa_1\kappa_2$-dimensional column vector and concatenating them together to form $\X$.
As illustrated in Figure~\ref{fig:conv-matmul} we may now express the convolution $\Yt = \Wt * \Xt$ as the matrix multiplication $\Y = \W \X$, where $\Yt \in \REAL^{f \times p_1 \times p_2}$ and $\Y \in \REAL^{f \times p}$ correspond to the tensor and matrix representation of the output feature maps, respectively, and $p_1$, $p_2$ denote the spatial dimensions of $\Yt$. The equivalence of $\Yt$ and $\Y$ can be easily established via an appropriate reshaping operation since $p = p_1 p_2$.

\begin{figure}
    \centering
    \includegraphics[width=0.8\textwidth, height=0.37\textwidth]{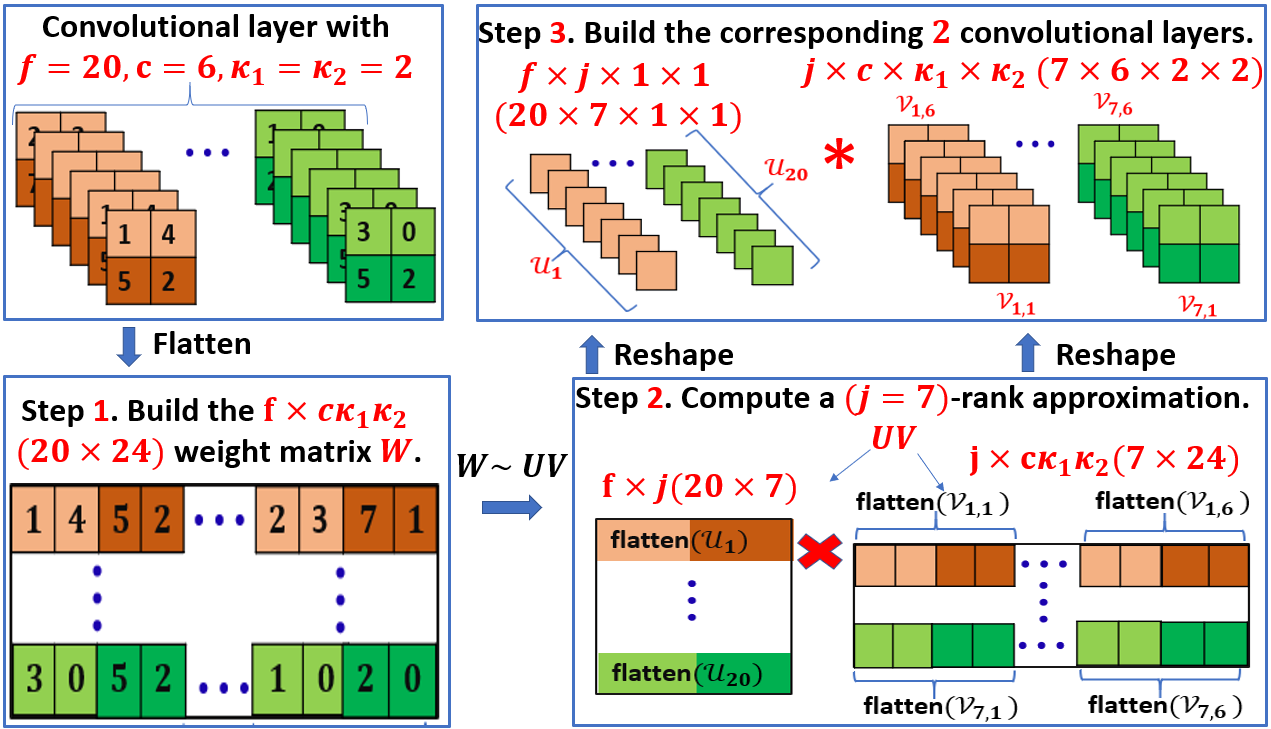}
    \caption{\small \textbf{Low-rank decomposition for convolutional layers via SVD.} The given convolution, c.f.~ Figure~\ref{fig:conv-matmul}, has $20$ filters, each of shape $6\times 2 \times 2$, resulting in a total of $480$ parameters. After extracting the corresponding weight matrix $W\in \REAL^{20\times 24}$ ($f \times c\kappa_1\kappa_2$), we compute its $(j=7)$-rank decomposition to obtain the pair of matrices $U\in \REAL^{20\times 7}$ ($f\times j$) and $V\in \REAL^{7 \times 24}$ ($j\times c\kappa_1\kappa_2$). Those matrices are encoded back as a pair of convolutional layers, the first (corresponding to $V$) has $j=7$ filters, $c=6$ channels and a $2\times2$ ($\kappa_1\times \kappa_2$) kernel, whereas the second (corresponding to $U$) is a $1\times1$ convolution of $f=20$ filters, and $j=7$ channels. The resulting layers have $308$ parameters.} \label{fig:conv-svd}
\end{figure}

\paragraph{Efficient tensor decomposition via SVD.}
Equipped with the notion of correspondence between convolution and matrix multiplication our goal is to decompose the layer via its matrix operator $W\in \REAL^{f \times c\kappa_1\kappa_2}$.
To this end, we compute the $j$-rank approximation of $W$ using SVD and factor it into a pair of smaller matrices $U\in \REAL^{f\times j}$ and $V\in \REAL^{j \times c\kappa_1\kappa_2 }$. More details on how to compute $U$ and $V$ are given in Section~\ref{sec-supp:method}.
We may then replace the original convolution, represented by $W$, by two smaller convolutions, represented by $V$ and $U$ for the first and second layer, respectively. Just like for the original layer, we can establish an equivalent convolution layer for both $U$ and $V$ as depicted in Figure~\ref{fig:conv-svd}.
To establish the equivalence we note that
(a) every row of the matrices $V$ and $U$ corresponds to a flattened filter of the respective convolution, and
(b) the number of channels in each layer is equal to the number of channels in its corresponding input tensor. %the number of channels in each layer corresponds to the number of filters in the previous layer.
% (c) and the number of filters ($f$) in the final output and the number of channels in the 
Hence, the first layer, which is represented by $V \in \REAL^{j \times c\kappa_1\kappa_2 }$ % or equivalently by its tensor representation $\Vt \in \REAL^{j \times c \times \kappa_1 \times \kappa_2}$, 
has $j$ filters, c.f.~(a), each consisting of $c$ channels, c.f.~(b), with kernel size $\kappa_1\times \kappa_2$. 
The second layer corresponding to $U$ has $f$ filters, c.f.~(a), $j$ channels, c.f.~(b), and a $1 \times 1$ kernel and may be equivalently represented as the tensor $\Ut \in \REAL^{f \times j \times 1 \times 1}$. Note that the number of weights is reduced from $fc\kappa_1\kappa_2$ to  $j(f+c\kappa_1\kappa_2)$.

\subsection{Clustering Methods}
\label{sec-supp:method_clustering}
As previously explained in Section~\ref{sec:method_layer}, one can cluster the columns of the corresponding weight matrix $W$, instead of clustering the channels of the convolutional layer. 
Here, the channel clustering can be defined as constraint clustering of these columns, where columns which include entries that correspond to the same kernel (e.g., the first $4$ columns in $W$ from Figure~\ref{fig:conv-svd}) are guaranteed to be in the same cluster.

This generalization is easily adaptable to other clustering methods that generate a wider set of solutions, e.g., the known \textbf{$k$-means}. An intuitive choice for our case is \textbf{projective clustering} and its variants. The goal of projective clustering is to compute a set of $k$ subspaces, each of dimension $j$, that minimizes the sum of \emph{squared} distances from each column in $W$ to its \emph{closest} subspace from this set. 
Then, we can partition the columns of $W$ into $k$ subsets according to their nearest subspace from this set. This is a natural extension of SVD that solves this problem for the case of $k=1$. However, this problem is known to be NP-hard, hence expensive approximation algorithms are required to solve it, or alternatively, a local minimum solution can be obtained using the Expectation-Maximization method (EM)~\citep{dempster1977maximum}. 

A more robust version of the previous method is to minimize the sum of \emph{non-squared} distances from the points to the subspaces. This approach is known to be more robust toward outliers (``far away points''). Similarly to the original variant (the case of squared distance), we can use the EM method to obtain a ``good guess'' for a solution of this problem, however, the EM method requires an algorithm that solves the problem for the case of $k=1$, i.e., computing the subspace that minimizes the sum of (non-squared) distances from those columns (for the sum of squared distances case, SVD is this algorithm). Unfortunately, there is only approximation algorithms~\citep{tukan2020compressed,clarkson2015input} for this case, and the deterministic versions are expensive in terms of running time.

Furthermore, and probably more importantly, all of these  methods cannot be considered as structured compression since arbitrary clustering may require re-shuffling the input tensors which could lead to significant slow-downs during inference. For example, when compressing a fully-connected layer, the arbitrary clustering may result in nonconsecutive neurons from the first layer that are connected to the same neuron in the second layer, while neurons that are between them are not. Hence, these layers can only have a large, sparse instead of a small, dense representation.

To this end, we choose to use \textbf{channel slicing}, i.e., we simply split the channels of the convolutional layer into $k$ chunks, where each chunk has at most $c/k$ consecutive channels. Splitting the channels into consecutive subsets (without allowing any arbitrary clustering) and applying the factorization on each one results in a structurally compressed layer without the need of special software/hardware support. Furthermore, this approach is the fastest among all the others.  % (since we split the channels into consecutive subsets without allowing any arbitrary clustering).
%However, these techniques cannot be considered as a structured pruning approach since the arbitrary clustering of the matrix columns may result in non-structurally pruned layers. 
%These techniques may give a very strong initial guess as they allow any arbitrary clustering the minimizes there cost function. 
% that may result in a better clustering. %However, it suffers from a strong disadvantage that is detailed next. %We now give a pair of appropriate clustering techniques that may be used. 
%For example, an intuitive clustering technique that is based on a generalization of the SVD factorization is 
%Here, the EM method requires a solution for the problem when $k=1$, i.e., computing the subspace that minimizes the sum of distances from those columns (for the sum of squared distances case, SVD is this algorithm). However, only approximation algorithms~\cite{tukan2020compressed,clarkson2015input} exist.  
% ths allows us to have a structured pruning technique after applying SVD on each subsets, while other approaches does not. 
%or this, in our experiments, we use this approach which is faster and results in a pruned smaller network, without the need of special software/hardware support. 
Finally, while other approaches may give a better initial guess for a compressed network in theory, in practice this is not the case; see Figures~\ref{fig-supp:ablation}. This is due to the global compression framework (Section~\ref{sec:method_network}) which repeatedly utilizes the channel slicing. 

We see that in practice, our method improve upon state-of-the-art techniques and obtains smaller networks with higher accuracy without the use of those complicated approaches that may result in sparse but not smaller network; see Section~\ref{sec-supp:results_ablation}.

\subsection{Compressing via SVD}
\label{sec-supp:method_compviasvd}
As explained in Section~\ref{sec:method_layer}, the weights of a convolutional layer (or a dense layer) can be encoded into a matrix $W\in \REAL^{f \times ck_1k_2}$, where $f$ and $c$ are the number of filters and channels in the layer, respectively, and $k_1\times k_2$ is the size of the kernel. 
In order to compress the matrix $W$ (and thus its corresponding layer) we aim to factor it into a pair of smaller matrices $U\in  \REAL^{f \times j}$ and $V \in \REAL^{j \times  ck_1k_2}$, such that $UV$ approximates the original matrix operator $W$. 

\paragraph{How to compute the matrices $U$ and $V$?}
For simplicity, let $d=ck_1k_2$.
We factor the matrix $W \in \REAL^{f \times d}$ via SVD to obtain $W = \tilde{U} \tilde{\Sigma}\tilde{V}$, where $ \tilde{U}$ is an ${f\times f}$ orthogonal matrix, $\tilde{\Sigma}$ is an ${\displaystyle f\times d}$ rectangular diagonal matrix with non-negative real numbers on the diagonal, and $\tilde{V}$ is an ${\displaystyle d \times d}$ orthogonal matrix. 

To compute a $j$-rank approximation of $W$, we can simply define the following $3$ matrices: $U\in \REAL^{f\times j}$, $\overline{V}\in \REAL^{j\times d}$, and $\Sigma \in \REAL^{j\times j}$, where $U$ is constructed by taking the first $j$ columns of $\tilde{U}$, $\overline{V}$  by taking the first $j$ rows of $\tilde{V}$, and $\Sigma$ is a diagonal matrix such that the $j$ entries on its diagonal are equal to the first $j$ entries on the diagonal of $\tilde{\Sigma}$. 
Now we have that ${U} {\Sigma}\overline{V}$ is the $j$-rank approximation of $W$. 

Finally, we can define $V=\Sigma\overline{V}$ to obtain a factorization of the $j$-rank approximation of $W$ as $UV$.

\subsection{Efficient Implementation of ALDS (Algorithm~\ref{alg:budget_allocation})} 
Our algorithm that is suggested in Section~\ref{sec:method_network} aims at minimizing the maximum relative error $\eps \coloneqq \max_{\ell \in [L]} \eps^\ell$ across the $L$ layers of the network as a proxy for the true cost, where $\eps^\ell$ is the theoretical maximum relative error in the $\ell^\text{th}$: 
$$\eps^\ell \coloneqq \frac{\norm{\What^\ell - \W^\ell}}{\norm{\W^\ell}}.$$

Through Algorithm~\ref{alg:budget_allocation}, for every $
\ell\in [L]$ we need to repeatedly compute $\eps^\ell$ as a function of $j^\ell$ and $k^\ell$. 
At Line~\ref{lin:opt_ranks}, we are given a guess for the optimal values of $k^1,\dots, k^L$, and our goal is to compute the values $j^1,\dots, j^L$ such that the resulting errors $\eps^1, \dots,\eps^L$ are (approximately) equal in order to minimize the maximum error $\max_{\ell \in [L]} \eps^\ell$ while achieving the desired global compression ratio. 
To this end, we guess a value for $\eps$ and for given $k^1, \ldots, k^L$ pick the corresponding $j^1, \ldots, j^L$ such that $\eps$ constitutes a tight upper bound for the relative error in each layer. Based on the now resulting budget (and consequently compression ratio) we can now improve our guess of $\eps$, e.g., via binary search or other types of root finding algorithms, until we convergence to a value of $\eps$ that corresponds to our desired overall compression ratio.
% Hence, an easy solution is to apply a binary search (or a similar root finding algorithm) over possible values of  the options of $j^1,\dots ,j^L$ that result in the desired overall prune ratio for the given values $k^1,\dots,k^L$, and for each option, compute the corresponding $\eps^1, \dots,\eps^L$. 
% Finally, we choose these with the smallest $\max_{\ell \in [L]} \eps^\ell$. 

Subsequently, for each layer we are given specific values of $k^\ell$ and $j^\ell$, which implies that we are given a budget $b^\ell$ for every layer $\ell\in[L]$.
Subsequently, we re-assign the number of subspaces $k^\ell$ and their ranks $j^\ell$ for each layer by iterating through the finite set of possible values for $k^\ell$ (Line~\ref{lin:opt_k}) and choosing the combination of $j^\ell$, $k^\ell$ that minimizes the relative error for the current layer budget $b^\ell$ (computed in Line~\ref{lin:layer_budget}). 

We then iteratively repeat both steps until convergence (Lines~\ref{lin:conv_start}-\ref{lin:conv_end}). 

Hence, instead of computing the cost of each layer at each step, we can save a lookup table that stores the errors $\eps^\ell$ for the possible values of $k^\ell$ and $j^\ell$ of each layer. For every layer $\ell \in [L]$, we iterate over the finite set of values of $k^\ell$, and we split the matrix $W^\ell$ to $k^\ell$ matrices (according to the channel slicing approach that is explained in Section~\ref{sec:method_layer}), then we compute the SVD factorization of each matrix from these $k^\ell$ matrices, and finally, compute $\eps^\ell$ that corresponds to a specific $j^\ell$ ($k^\ell$ is already given) in $O(fd)$ time, where $f$ is the number of rows in the weight matrix that corresponds to the $\ell$th layer and $d$ is the number of columns. 

Furthermore, instead of computing each option of $\eps^\ell$ in $O(fd)$ time, we use the upper bound derived in Theorem~\ref{thm:rel_error} to compute it in $O(k)$ time and saving it in the lookup table. Specifically, we can express the relative error as a function of the rank and we thus only need to solve the underlying SVD for each layer once for each value of $k^\ell$. Without Theorem~\ref{thm:rel_error} we would need to compute the relative error (operator norm) for each pair $j^\ell, k^\ell$ separately. This would in turn result in a significant slowdown of the runtime of Algorithm~\ref{alg:budget_allocation}. 
Hence, the combined use of a look-up table and the application of Theorem~\ref{thm:rel_error} ensures a more efficient implementation of Algorithm~\ref{alg:budget_allocation}.

\subsection{Additional Discussion of ALDS (Algorithm~\ref{alg:budget_allocation})}

Below, we include additional details and clarification regarding Algorithm~\ref{alg:budget_allocation}.

\subsubsection*{Overview}
At a high-level, Algorithm~\ref{alg:budget_allocation} aims to find a local optimum for the optimization procedure described in Equation~\eqref{eq:budget_allocation}. We hereby iteratively optimize for $k^1, \ldots, k^L$ and $j^1, \ldots, j^L$. The step where we fix the set of $k$’s and optimize for the set of $j$’s is Line~\ref{lin:opt_ranks}, whereas in Line~\ref{lin:opt_k} we fix the layer budget and optimize for the set of $k$’s. At each step the objective is minimized. Thus for a fixed seed, ALDS converges to a local optimum of~\eqref{eq:budget_allocation}. We then repeat the entire procedure multiple times with different random seeds to improve the quality of the local optimum. 

\subsubsection*{Line~\ref{lin:opt_ranks}: $\textsc{OptimalRanks}(\mathbf{CR}, k^1, \ldots, k^L)$}
At this step we are given a guess for the optimal values of $k^1,\dots, k^L$, and our goal is to compute the values $j^1,\dots, j^L$ that minimize the objective function described in Equation~\eqref{eq:budget_allocation}, i.e., the maximum error $\max_{\ell \in [L]} \eps^\ell$, while achieving the desired global compression ratio $\mathbf{CR}$.

To find the optimal solution, we note the following. Recall that $k$’s are fixed.

\begin{enumerate}
    \item 
    \textbf{The maximum error is minimized exactly when all errors are equal.} To see that this is indeed the case we can proceed by contradiction. Suppose we found an optimal solution where all errors are not equal. Then we could use some of our compression budget to add more parameters to the layer with the maximum error while removing the same amount of parameters from the layer with minimum error. Since adding more parameters improves the error we just lowered the maximum error by adding more parameters to the layer with the maximum error. Hence, this leads to a contradiction proving our initial statement.
    \item
        \textbf{A given constant error across layers corresponds to a fixed compression ratio.} This should be very straightforward to see. Specifically, for a given layer error we can find the corresponding rank and the rank implies how many parameters the compressed layer will have. This then implies a fixed compression ratio. Moreover, note that this relation is monotonic.
\end{enumerate}

Both (1.) and (2.) together imply that we can use a binary search or some other root finding algorithm to determine the corresponding constant error for a desired compression ratio $\textsc{OptimalRanks}$. The solution of our binary search will then be the corresponding set of $j$’s (ranks) for each layer that minimizes the maximum error for a desired compression ratio and given set of $k$’s (recall that we optimize for $k$’s separately).

\subsubsection*{Line~\ref{lin:opt_k}: $\textsc{OptimalSubspaces}(b^\ell)$}

This step is fairly straightforward to follow. First, we note that for this step we proceed on a per-layer basis. Here, for each layer we are given specific values of $k^\ell$ and $j^\ell$, which implies that we are given a budget $b^\ell$ for every layer $\ell\in[L]$.
Subsequently, we re-assign the number of subspaces $k^\ell$ and their ranks $j^\ell$ for each layer $\ell$ as follow: 
We iterate through the finite set of possible values for $k^\ell$, for every such value $k^\ell$ we pick its corresponding $j^\ell$ such that the total size (number of parameters) of this layer is (approximately) the given budget $b^\ell$. 
Now, for every pair of candidates $k^\ell$ and $j^\ell$ we compute the relative error on this layer that is caused after compression with respect to these values. 
Finally, we choose the combination of $j^\ell$, $k^\ell$ that minimizes the relative error for the current layer budget $b^\ell$. We then discard the values found for $j^1, \ldots, j^L$ and re-optimize them in the next iteration of $\textsc{OptimalRanks}$.

Note that for $\textsc{OptimalSubspaces}$ there is no monotonic relation between the value of $k^\ell$ and the corresponding error like there is between the value of $j^\ell$ and the error. Hence, we proceed on a per-layer basis where we keep the per-layer budget constant during $\textsc{OptimalSubspaces}$ as described above.

\subsubsection*{Optimality}

From the details of the two steps, it should be very clear that the cost is decreasing at each step in the optimization procedure and we can thus conclude that for each random seed Algorithm~\ref{alg:budget_allocation} converges to a local optimum (at which point the cost will be non-increasing).

\subsubsection*{Even More Remarks}

Note that above for $\textsc{OptimalRanks}$ we assumed that the errors ($\varepsilon^\ell$’s) are continuous but they are actually discrete given that they are a function of the rank which is discrete. However, as long as we can ensure that the objective decreases at every iteration we can still reach a local minimum.

Alternatively, we can solve the continuous relaxation of the above problem and use a randomized rounding\footnote{\url{https://en.wikipedia.org/wiki/Randomized_rounding}} approach to get an approximately optimal solution.

In practice, however, we found that it is not necessary to add this additional complication step since it is sufficient that the cost objective decreases at every time step and we cannot hope to obtain a global optimum anyway (we can only approximate it with repeating the optimization procedure with multiple random seeds, which we do, see Algorithm~\ref{alg:budget_allocation}).

\subsection{Extensions of ALDS}
\label{sec-supp:method_extension}
As mentioned in Section~\ref{sec:method_alds} ALDS can be readily extended to any desired \emph{set} of low-rank compression techniques. Specifically, we can replace the local step of Line~\ref{lin:opt_k} by a search over different methods, e.g., Tucker decomposition, PCA, or other SVD compression schemes, and return the best method for a given budget. In general, we may combine ALDS with any low-rank compression as long as we can efficiently evaluate the per-layer error of the compression scheme. Note that this essentially equips us with a framework to automatically choose the per-layer decomposition technique fully automatically.

To this end, we test an extension of ALDS where in addition to searching over multiple values of $k^\ell$ we simultaneously search over various flattening schemes to convert a convolutional tensor to a matrix before applying SVD.

As before, let $\Wt \in \REAL^{f \times c \times \kappa_1 \times \kappa_2}$ denote the weight tensor for a convolutional layer with $f$ filters, $c$ input channels, and a $\kappa_1 \times \kappa_2$ kernel. Moreover, let $j$ denote the desired rank of the decomposition.
We consider the following schemes to automatically search over: 

\begin{itemize}[itemsep=5pt, leftmargin=25pt, labelwidth=10pt, topsep=5pt, partopsep=0pt]
    \item 
    \textsc{Scheme 0}: flatten the tensor to a matrix of shape $f \times c \kappa_1 \kappa_2$. The decomposed layers correspond to a $j \times c \times \kappa_1 \times \kappa_2$-convolution followed by a $f \times j \times 1 \times 1$-convolution. This is the same scheme as used in ALDS.
    \item 
    \textsc{Scheme 1}: flatten the tensor to a matrix of shape $f \kappa_1 \times c \kappa_2$. The decomposed layers correspond to a $j \times c \times 1 \times \kappa_2$-convolution followed by a $f \times j \times \kappa_1 \times 1$-convolution.
    \item
    \textsc{Scheme 2}: flatten the tensor to a matrix of shape $f \kappa_2 \times c \kappa_1$. The decomposed layers correspond to a $j \times c \times \kappa_1 \times 1$-convolution followed by a $f \times j \times 1 \times \kappa_2$-convolution.
    \item
    \textsc{Scheme 3:}
    flatten the tensor to a matrix of shape $f \kappa_1 \kappa_2 \times c$. The decomposed layers correspond to a $j \times c \times 1 \times 1$-convolution followed by a $f \times j \times \kappa_1 \times \kappa_2$-convolution.
\end{itemize}

We denote this method by ALDS+ and provide preliminary results in Section~\ref{sec-supp:results_extension}. We note that since ALDS+ is a generalization of ALDS its performance is at least as good as the original ALDS. Moreover, our preliminary results actually suggest that the extension clearly improves upon the empirical performance of ALDS.
\section{Experimental Setup and Hyperparameters}
\label{sec-supp:setup}
Our experimental evaluations are based on a variety of network architectures, data sets, and compression pipelines. In the following, we provide all necessary hyperparameters to reproduce our experiments for each of the datasets and respective network architectures. 

All networks were trained, compressed, and evaluated on a compute cluster with NVIDIA Titan RTX and NVIDIA RTX 2080Ti GPUs. The experiments were conducted with PyTorch 1.7 and our code is fully open-sourced~\footnote{Code repository: \url{https://github.com/lucaslie/torchprune}}.

All networks are trained according to the hyperparameters outlined in the respective original papers. During retraining, which is described in Section~\ref{sec-supp:setup-pipeline}, we reuse the same hyperparameters.

Moreover, each experiment is repeated 3 times and we report mean and mean, standard deviation in the tables and figures, respectively. 

For each data set, we use the publicly available development set as test set and use a 90\%/5\%/5\% split on the train set to obtain a separate train and \emph{two} validation sets. One validation set is used for data-dependent compression methods, e.g., PCA~\citep{Zhang2015a}; the other set is used for early stopping during training.

\subsection{Experimental Setup for CIFAR10}
\label{sec-supp:setup-cifar10}
All relevant hyperparameters are outlined in Table~\ref{tab-supp:cifar_hyperparameters}.
For each of the networks we use the training hyperparameters outlined in the respective original papers, i.e., as described by \citet{he2016deep}, \citet{simonyan2014very}, \citet{huang2017densely}, and \citet{zagoruyko2016wide} for ResNets, VGGs, DenseNets, and WideResNets (WRN), respectively. 

We add a warmup period in the beginning where we linearly scale up the learning rate from 0 to the nominal learning rate to ensure proper training performance in distributed training settings~\citep{goyal2017accurate}. 

During training we use the standard data augmentation strategy for CIFAR: (1) zero padding from 32x32 to 36x36; (2) random crop to 32x32; (3) random horizontal flip; (4) channel-wise normalization. During inference only the normalization (4) is applied.

The compression ratios are chosen according to a geometric sequence with the common ratio denoted by $\alpha$ in Table~\ref{tab-supp:cifar_hyperparameters}, i.e., the compression ratio for iteration $i$ is determined by $1 - \alpha^{i}$. 
The compression parameter $n_\text{seed}$ denotes the number of seeds used to initialize Algorithm~\ref{alg:budget_allocation} for compressing with PP. 

\begin{table}
%\small
\centering
\caption{
The experimental hyperparameters for training, compression, and retraining for the tested \textbf{CIFAR10} network architectures.
``LR'' and ``LR decay'' hereby denote the learning and the (multiplicative) learning rate decay, respectively, that is deployed at the epochs as specified. ``$\{x, \ldots\}$'' indicates that the learning rate is decayed every $x$ epochs.}
\label{tab-supp:cifar_hyperparameters}
\begin{adjustbox}{width=1.0\textwidth}
\begin{tabular}{|c|c|l||c|c|c|c|}
\hline
\multirow{14}{*}{\rotatebox{90}{CIFAR10}}
& \multicolumn{2}{c||}{Hyperparameters}
& VGG16 & Resnet20 & DenseNet22 & WRN-16-8 \\ \cline{2-7}
& \multirow{9}{*}{(Re-)Training} 
& Test accuracy (\%) &  92.81         & 91.4
                     &  89.90         & 95.19 \\
& & Loss          & cross-entropy  & cross-entropy
                  & cross-entropy  & cross-entropy \\
& & Optimizer     & SGD            & SGD 
                  & SGD            & SGD  \\
& & Epochs        & 300            & 182
                  & 300            & 200  \\
& & Warm-up       & 10             & 5    
                  & 10             & 5    \\
& & Batch size    & 256            & 128
                  & 64             & 128   \\
& & LR            & 0.05           & 0.1
                  & 0.1            & 0.1  \\
& & LR decay      & 0.5@\{30, \ldots\} & 0.1@\{91, 136\}
                  & 0.1@\{150, 225\}   & 0.2@\{60, \ldots\}\\
& & Momentum      & 0.9            & 0.9
                  & 0.9            & 0.9  \\
& & Nesterov      & \xmark         & \xmark      
                  & $\checkmark$   & $\checkmark$ \\
& & Weight decay  & 5.0e-4         & 1.0e-4 
                  & 1.0e-4         & 5.0e-4   \\ \cline{2-7}
& \multirow{2}{*}{Compression}
&   $\alpha$      & 0.80           & 0.80
                  & 0.80           & 0.80 \\ 
&& $n_\text{seed}$& 15             & 15
                  & 15             & 15  \\ \hline
\end{tabular}
\end{adjustbox}
\end{table}

\subsection{Experimental Setup for ImageNet}
\label{sec-supp:setup-imagenet}
We report the relevant hyperparameters in Table~\ref{tab-supp:imagenet_hyperparameters}. For ImageNet we consider the networks architectures Resnet18~\citep{he2016deep}, AlexNet~\citep{Alex2012}, and MobileNetV2~\citep{sandler2018mobilenetv2}. 

During training we use the following data augmentation: (1) randomly resize and crop to 224x224; (2) random horizontal flip; (3) channel-wise normalization. During inference, we use a center crop to 224x224 before (3) is applied.

Note that for MobileNetV2 we deploy a lower initial learning rate during retraining. Otherwise, all hyperparameters remain the same during retraining.

\begin{table}
%\small
\centering
\caption{
The experimental hyperparameters for training, compression, and retraining for the tested \textbf{ImageNet} network architectures.
``LR'' and ``LR decay'' hereby denote the learning and the (multiplicative) learning rate decay, respectively, that is deployed at the epochs as specified. ``$\{x, \ldots\}$'' indicates that the learning rate is decayed every $x$ epochs.}
\label{tab-supp:imagenet_hyperparameters}
\begin{adjustbox}{width=1.0\textwidth}
\begin{tabular}{|c|c|l||c|c|c|}
\hline
\multirow{15}{*}{\rotatebox{90}{ImageNet}}
& \multicolumn{2}{c||}{Hyperparameters}
& ResNet18 & AlexNet & MobileNetV2 \\ \cline{2-6}
& \multirow{9}{*}{(Re-)Training} 
&   Top 1 Test accuracy (\%) &  69.64 & 57.30 & 71.85  \\
& & Top 5 Test accuracy (\%) &  88.98 & 80.20 & 90.33  \\
& & Loss          & cross-entropy  & cross-entropy & cross-entropy \\
& & Optimizer     & SGD            & SGD & RMSprop \\
& & Epochs        & 90             & 90  & 300 \\
& & Warm-up       & 5              & 5   & 0 \\
& & Batch size    & 256            & 256 & 768 \\
& & LR            & 0.1            & 0.1 & 0.045 (1e-4) \\
& & LR decay      & 0.1@\{30, 60, 80\} &  0.1@\{30, 60, 80\} & 0.98 per step \\
& & Momentum      & 0.9            &  0.9 & 0.9 \\
& & Nesterov      & \xmark         & \xmark & \xmark \\
& & Weight decay  & 1.0e-4         & 1.0e-4 & 4.0e-5 \\ \cline{2-6}
& \multirow{2}{*}{Compression}
&   $\alpha$      & 0.80           & 0.80 & 0.80 \\ 
&& $n_\text{seed}$& 15             & 15  & 15 \\ \hline
\end{tabular}
\end{adjustbox}
\end{table}

\subsection{Experimental Setup for Pascal VOC}
\label{sec-supp:setup-voc}
In addition to CIFAR and ImageNet, we also consider the segmentation task from Pascal VOC 2012~\cite{everingham2015pascal}. We augment the nominal data training data using the extra labels as provided by~\citet{hariharan2011semantic}. As network architecture we consider a DeeplabV3~\cite{chen2017rethinking} with ResNet50 backbone pre-trained on ImageNet. 

During training we use the following data augmentation pipeline: (1) randomly resize (256x256 to 1024x1024) and crop to 513x513; (2) random horizontal flip; (3) channel-wise normalization. During inference, we resize to 513x513 exactly before the normalization (3) is applied. 

We report both intersection-over-union (IoU) and Top1 test accuracy for each of the compressed and uncompressed networks. The experimental hyperparameters are summarized in Table~\ref{tab-supp:voc_hyperparameters}.

\begin{table}
%\small
\centering
\caption{
The experimental hyperparameters for training, compression, and retraining for the tested \textbf{VOC} network architecture.
``LR'' and ``LR decay'' hereby denote the learning and the learning rate decay, respectively. Note that the learning rate is polynomially decayed after each step.}
\label{tab-supp:voc_hyperparameters}
\begin{tabular}{|c|c|l||c|}
\hline
\multirow{15}{*}{\rotatebox{90}{Pascal VOC 2012 -- Segmentation}}
& \multicolumn{2}{c||}{Hyperparameters}
& DeeplabV3-ResNet50 \\ \cline{2-4}
& \multirow{9}{*}{(Re-)Training} 
&   IoU Test accuracy (\%) &  69.84 \\
& & Top 1 Test accuracy (\%) &  94.25 \\
& & Loss          & cross-entropy  \\
& & Optimizer     & SGD            \\
& & Epochs        & 45             \\
& & Warm-up       & 0              \\
& & Batch size    & 32            \\
& & LR            & 0.02            \\
& & LR decay      & $\text{(1 - ``step''/``total steps'')} ^ \text{0.9}$ \\
& & Momentum      & 0.9            \\
& & Nesterov      & \xmark         \\
& & Weight decay  & 1.0e-4         \\ \cline{2-4}
& \multirow{2}{*}{Compression}
&   $\alpha$      & 0.80           \\ 
&& $n_\text{seed}$& 15             \\ \hline
\end{tabular}
\end{table}

\subsection{Baseline Methods}
\label{sec-supp:setup-comparisons}
We implement and compare against the following compression methods for our baseline experiments: 

\begin{enumerate}[itemsep=5pt, leftmargin=25pt, labelwidth=10pt, topsep=5pt, partopsep=0pt]
    \item[\large 1.]
    \textsc{PCA}~\citep{Zhang2015a} decomposes each layer based on principle component analysis of the pre-activation (output of linear layer). We implement the symmetric, linear version of their method. The per-layer compression ratio is based on the greedy solution for minimizing the product of the per-layer energy, where the energy is defined as the sum of singular values in the compressed layer, see Equation (14) of~\citet{Zhang2015a}.
    \item[\large 2.]
    \textsc{SVD-Energy}~\citep{alvarez2017compression,Wen2017} decomposes each layer via matrix folding akin to our SVD-based decomposition. The per-layer compression ratio is found by keeping the relative energy reduction constant across layers, where energy is defined as the sum of squared singular values. 
    \item[\large 3.]
    \textsc{SVD}~\citep{Denton2014} decomposes each layer via matrix folding akin to our SVD-based decomposition. However, we hereby fix $k^\ell=1$ for all layers $\ell \in [L]$ in order to provide a nominal comparison akin of ``standard'' tensor decomposition. The per-layer compression ratio is kept constant across all layers.
    \item[\large 4.]
    \textsc{L-Rank}~\citep{idelbayev2020low} decomposes each layer via matrix folding akin to our SVD-based decomposition. The per-layer compression is determined by minimizing a joint cost objective of the energy and the computational cost of each layer, see Equation (5) of~\citet{idelbayev2020low} for details.
    \item[\large 5.] 
    \textsc{FT}~\cite{li2016pruning} prunes the filters (or neurons) in each layer with the lowest element-wise $\ell_2$-norm. The per-layer compression ratio is set manually (constant in our implementation).
    \item[\large 6.]
    \textsc{PFP}~\cite{liebenwein2020provable} prunes the channels with the lowest sensitivity, where the data-dependent sensitivities are based on a provable notion of channel pruning. The per-layer prune ratio is determined based on the associated theoretical error guarantees.
    % \item[\large 2.]
    % \textsc{SoftNet}~\cite{he2018soft} prunes the filters (or neurons) in each layer with the lowest element-wise $\ell_1$-norm. The per-layer compression ratio is chosen such that the number of filters is reduced by the same ratio in each layer. 
    % \item[\large 3.]
    % \textsc{ThiNet}~\cite{luo2017thinet} greedily ranks channels according to their contribution to the pre-activation of the layer's output based on a small batch of data points. The per-layer compression ratio is determined manually.
\end{enumerate}

\subsection{Compress-Retrain Pipeline}
\label{sec-supp:setup-pipeline}

Recall that our baseline experiments are based on the following unified compress-retrain pipeline across all compression methods: 

\begin{enumerate}[itemsep=5pt, leftmargin=25pt, labelwidth=10pt, topsep=5pt, partopsep=0pt]
    \item[\large 1.] 
    \textsc{Train} for $e$ epochs according to the standard training schedule for the respective network. 
    \item[\large 2.]  
    \textsc{Compress} the network according to the chosen method.
    \item[\large 3.] 
    \textsc{Retrain} the network for the desired amount of $r$ epochs using the original training hyperparameters from the epochs in the range $[e-r,e]$. 
    \item[\large 4.]
    \textsc{Iteratively} repeat  1.-3. after projecting the decomposed layers back (\textbf{optional}).
\end{enumerate}

In addition, we also consider experiments in the iterative learning rate rewinding setting, where steps 2 and 3 are repeated iteratively (optional step 4). 

While various papers combine their compression methods with different retrain schedules we unify the compress-retrain pipeline across all tested methods for our baseline experiments to ensure that results are comparable. Note that the implemented compress-retrain pipeline as originally introduced by~\citet{renda2020comparing} has been shown to yield consistently good compression results across various compression/pruning setups (unstructured, structured) and tasks (computer vision, NLP). Hence, we choose to concentrate on that particular pipeline. 

\FloatBarrier
\section{Additional Experimental Results}
\label{sec-supp:results}
In this section, we provide additional results of our experimental evaluations.

\subsection{Complete Tables for One-shot Compression Experiments from Section~\ref{sec:experiments_oneshot}}

\begin{table}[htb]
\begingroup
\setlength{\tabcolsep}{3.05pt} % Default value: 6pt
\renewcommand{\arraystretch}{1.2} % Default value: 1
\centering
\caption{
The maximal compression ratio for which the drop in test accuracy is at most some pre-specified $\delta$ on CIFAR10. The table reports compression ratio in terms of parameters and FLOPs, denoted by CR-P and CR-F, respectively.
When the desired $\delta$ was not achieved for any compression ratio in the range the fields are left blank. The top values achieved for CR-P and CR-F are bolded.}
\label{tab-supp:cifar_retrain}
\begin{adjustbox}{width=1\textwidth}
\begin{tabular}{|c|c|c||ccc|ccc|ccc|ccc|ccc|}
\hline
\multirow{30}{*}{\rotatebox{90}{CIFAR10}}
& \multirow{2}{*}{Model}
& \multirow{2}{*}{\shortstack{Prune \\ Method}}
& \multicolumn{3}{c|}{$\delta=0.0\%$}
& \multicolumn{3}{c|}{$\delta=0.5\%$}
& \multicolumn{3}{c|}{$\delta=1.0\%$}
& \multicolumn{3}{c|}{$\delta=2.0\%$}
& \multicolumn{3}{c|}{$\delta=3.0\%$} \\
& &
& Top1 Acc. & CR-P & CR-F
& Top1 Acc. & CR-P & CR-F
& Top1 Acc. & CR-P & CR-F
& Top1 Acc. & CR-P & CR-F
& Top1 Acc. & CR-P & CR-F \\ \cline{2-18}
& \multirow{7}{*}{\shortstack{ResNet20 \\ \\ Top1: 91.39}}
& ALDS
& +0.09 & \textbf{64.58} & \textbf{55.95}
& -0.47 & \textbf{74.91} & \textbf{67.86}
& -0.68 & \textbf{79.01} & \textbf{71.59}
& -1.88 & \textbf{87.68} & \textbf{83.23}
& -2.59 & \textbf{89.65} & \textbf{85.32} \\
& & PCA
& +0.16 & 39.98 & 38.64
& -0.11 & 49.88 & 48.67
& -0.58 & 58.04 & 57.21
& -1.41 & 70.54 & 70.78
& -2.11 & 75.23 & 76.01 \\
& & SVD-Energy
& +0.14 & 40.22 & 39.38
& -0.21 & 49.88 & 49.08
& -0.83 & 57.95 & 57.15
& -1.52 & 64.76 & 64.10
& -2.17 & 70.47 & 70.01 \\
& & SVD
& +0.24 & 14.36 & 15.34
& -0.29 & 39.81 & 38.95
& -0.90 & 49.19 & 50.21
& -1.08 & 57.47 & 57.80
& -2.88 & 70.14 & 71.31 \\
& & L-Rank
& +0.14 & 15.00 & 29.08
& -0.44 & 28.71 & 54.89
& -0.44 & 28.71 & 54.89
& -1.56 & 49.87 & 72.57
& -2.82 & 64.81 & 80.80 \\
& & FT
& +0.15 & 15.29 & 16.66
& -0.32 & 39.69 & 39.57
& -0.75 & 57.77 & 55.85
& -1.88 & 74.89 & 71.76
& -2.71 & 79.29 & 76.74 \\
& & PFP
& +0.12 & 28.74 & 20.56
& -0.28 & 40.28 & 30.06
& -0.85 & 58.26 & 46.94
& -1.56 & 70.49 & 59.78
& -2.57 & 79.28 & 69.27 \\
\cline{2-18}
& \multirow{7}{*}{\shortstack{VGG16 \\ \\ Top1: 92.78}}
& ALDS
& +0.29 & \textbf{94.89} & \textbf{83.94}
& -0.11 & \textbf{95.77} & \textbf{86.23}
& -0.52 & \textbf{97.01} & \textbf{88.95}
& -0.52 & \textbf{97.03} & 88.95
& -0.52 & \textbf{97.03} & 88.95 \\
& & PCA
& +0.47 & 87.74 & 81.05
& -0.02 & 89.72 & 85.84
& -0.02 & 89.72 & 85.84
& -1.12 & 91.37 & 89.57
& -1.12 & 91.37 & 89.57 \\
& & SVD-Energy
& +0.35 & 79.21 & 78.70
& -0.08 & 82.57 & 81.32
& -0.83 & 87.74 & 85.36
& -1.22 & 89.71 & 87.13
& -2.08 & 91.37 & 88.58 \\
& & SVD
& +0.29 & 70.35 & 70.13
& +0.29 & 70.35 & 70.13
& -0.74 & 75.18 & 75.13
& -1.58 & 82.58 & 82.39
& -1.58 & 82.58 & 82.39 \\
& & L-Rank
& +0.35 & 82.56 & 69.67
& -0.35 & 85.38 & 75.86
& -0.35 & 85.38 & 75.86
& -0.35 & 85.38 & 75.86
& -0.35 & 85.38 & 75.86 \\
& & FT
& +0.17 & 64.81 & 62.16
& -0.47 & 79.13 & 78.44
& -0.87 & 82.61 & 82.41
& -1.95 & 89.69 & 89.91
& -2.66 & 91.35 & \textbf{91.68} \\
& & PFP
& +0.16 & 89.73 & 74.61
& -0.47 & 94.87 & 84.76
& -0.96 & 96.40 & 88.38
& -1.33 & 97.02 & \textbf{90.25}
& -1.33 & 97.02 & 90.25 \\
\cline{2-18}
& \multirow{7}{*}{\shortstack{DenseNet22 \\ \\ Top1: 89.88}}
& ALDS
& +0.17 & \textbf{48.85} & \textbf{51.90}
& -0.32 & \textbf{56.84} & \textbf{61.98}
& -0.54 & \textbf{63.83} & \textbf{69.68}
& -1.87 & \textbf{69.67} & \textbf{74.48}
& -1.87 & \textbf{69.67} & \textbf{74.48} \\
& & PCA
& +0.20 & 14.67 & 34.55
& +0.20 & 14.67 & 34.55
& -0.73 & 28.83 & 57.02
& -0.73 & 28.83 & 57.02
& -2.75 & 40.51 & 70.03 \\
& & SVD-Energy
&   &   &  
& -0.29 & 15.16 & 19.34
& -0.29 & 15.16 & 19.34
& -1.28 & 28.62 & 33.26
& -2.21 & 40.20 & 44.72 \\
& & SVD
& +0.13 & 15.00 & 15.33
& +0.13 & 15.00 & 15.33
& -0.87 & 26.73 & 27.41
& -0.87 & 26.73 & 27.41
& -2.51 & 37.99 & 39.25 \\
& & L-Rank
& +0.26 & 14.98 & 35.21
& +0.26 & 14.98 & 35.21
& -0.90 & 28.67 & 63.55
& -1.82 & 40.33 & 73.45
& -1.82 & 40.33 & 73.45 \\
& & FT
& +0.15 & 15.49 & 16.70
& -0.24 & 28.33 & 29.50
& -0.24 & 28.33 & 29.50
& -1.46 & 51.10 & 51.03
& -2.40 & 64.12 & 63.09 \\
& & PFP
& +0.00 & 28.68 & 32.60
& -0.44 & 40.24 & 43.37
& -0.70 & 49.67 & 51.94
& -1.36 & 58.20 & 58.21
& -2.43 & 65.17 & 64.50 \\
\cline{2-18}
& \multirow{7}{*}{\shortstack{WRN16-8 \\ \\ Top1: 95.21}}
& ALDS
& +0.05 & 28.67 & 13.00
& -0.42 & \textbf{87.77} & 79.90
& -0.88 & \textbf{92.75} & 87.39
& -1.53 & \textbf{95.69} & 92.50
& -2.23 & \textbf{97.01} & \textbf{95.51} \\
& & PCA
& +0.14 & 15.00 & 7.98
& -0.49 & 85.33 & \textbf{83.45}
& -0.96 & 91.33 & \textbf{90.23}
& -1.76 & 93.90 & \textbf{93.15}
& -2.45 & 94.87 & 94.30 \\
& & SVD-Energy
& +0.29 & 15.01 & 6.92
& -0.41 & 64.75 & 60.94
& -0.81 & 85.38 & 83.52
& -1.90 & 91.38 & 90.04
& -2.46 & 92.77 & 91.58 \\
& & SVD
&   &   &  
&   &   &  
& -0.96 & 40.20 & 39.97
& -1.63 & 70.48 & 70.49
& -1.63 & 70.48 & 70.49 \\
& & L-Rank
& +0.25 & 14.99 & 6.79
& -0.45 & 49.86 & 58.00
& -0.88 & 75.20 & 82.26
& -1.70 & 87.73 & 92.03
& -2.18 & 89.72 & 93.51 \\
& & FT
& +0.03 & \textbf{64.54} & \textbf{61.53}
& -0.32 & 82.33 & 75.97
& -0.95 & 89.70 & 83.52
& -1.78 & 94.91 & 90.82
& -2.86 & 96.42 & 93.33 \\
& & PFP
& +0.05 & 57.92 & 54.74
& -0.44 & 85.33 & 80.68
& -0.77 & 89.71 & 85.16
& -1.69 & 95.65 & 92.60
& -2.40 & 96.96 & 94.36 \\
\hline
\end{tabular}

\end{adjustbox}
% \vspace{3px}
\endgroup
\end{table}

\begin{table}[htb]
\begingroup
\setlength{\tabcolsep}{3.05pt} % Default value: 6pt
\renewcommand{\arraystretch}{1.3} % Default value: 1
\centering
\caption{
The maximal compression ratio for which the drop in test accuracy is at most $\delta=1.0\%$ for ResNet20 (CIFAR10) for various amounts of retraining (as indicated).
The table reports compression ratio in terms of parameters and FLOPs, denoted by CR-P and CR-F, respectively.
When the desired $\delta$ was not achieved for any compression ratio in the range the fields are left blank. The top values achieved for CR-P and CR-F are bolded.
}
\label{tab-supp:cifar_retrainsweep}
\begin{adjustbox}{width=1\textwidth}
\begin{tabular}{|c|c|c||ccc|ccc|ccc|ccc|ccc|ccc|}
\hline
\multirow{9}{*}{\rotatebox{90}{CIFAR10}}
& \multirow{2}{*}{Model}
& \multirow{2}{*}{\shortstack{Prune \\ Method}}
& \multicolumn{3}{c|}{$r=0\%\,e$}
& \multicolumn{3}{c|}{$r=5\%\,e$}
& \multicolumn{3}{c|}{$r=10\%\,e$}
& \multicolumn{3}{c|}{$r=25\%\,e$}
& \multicolumn{3}{c|}{$r=50\%\,e$}
& \multicolumn{3}{c|}{$r=100\%\,e$} \\
& &
& Top1 Acc. & CR-P & CR-F
& Top1 Acc. & CR-P & CR-F
& Top1 Acc. & CR-P & CR-F
& Top1 Acc. & CR-P & CR-F
& Top1 Acc. & CR-P & CR-F
& Top1 Acc. & CR-P & CR-F \\ \cline{2-21}
& \multirow{7}{*}{\shortstack{ResNet20 \\ \\ Top1: 91.39}}
& ALDS
& -0.13 & \textbf{14.82} & \textbf{7.03}
& -0.53 & \textbf{35.87} & \textbf{26.27}
& -0.73 & \textbf{43.12} & \textbf{33.65}
& -0.65 & \textbf{43.14} & \textbf{33.33}
& -0.86 & \textbf{62.39} & \textbf{54.40}
& -0.88 & \textbf{81.29} & \textbf{74.23} \\
& & PCA
&   &   &  
&   &   &  
& -0.74 & 19.31 & 18.64
& -0.70 & 19.34 & 18.44
& -0.59 & 36.21 & 35.19
& -0.74 & 60.29 & 59.81 \\
& & SVD-Energy
&   &   &  
& -0.64 & 14.99 & 14.09
& -0.70 & 19.61 & 18.81
& -0.59 & 19.61 & 18.81
& -0.73 & 43.46 & 42.49
& -0.46 & 55.25 & 54.59 \\
& & SVD
&   &   &  
&   &   &  
& -0.83 & 14.36 & 15.34
& -0.58 & 14.36 & 15.34
& -0.69 & 28.21 & 29.11
& -0.77 & 51.58 & 51.52 \\
& & L-Rank
&   &   &  
&   &   &  
&   &   &  
& -0.64 & 15.00 & 29.08
& -0.33 & 15.00 & 29.08
& -0.44 & 28.71 & 54.89 \\
& & FT
&   &   &  
&   &   &  
&   &   &  
& -0.67 & 15.29 & 16.66
& -0.69 & 27.76 & 28.40
& -0.75 & 57.77 & 55.85 \\
& & PFP
&   &   &  
&   &   &  
&   &   &  
& -0.77 & 14.88 & 9.61
& -0.83 & 32.71 & 23.85
& -0.54 & 52.89 & 42.04 \\
\hline
\end{tabular}

\end{adjustbox}
\endgroup
\end{table}

\begin{table}[t!]
\begingroup
\setlength{\tabcolsep}{3.05pt} % Default value: 6pt
\renewcommand{\arraystretch}{1.2} % Default value: 1
\centering
\caption{
The maximal compression ratio for which the drop in test accuracy is at most some pre-specified $\delta$ on ResNet18 and MobileNetV2 (both ImageNet). The table reports compression ratio in terms of parameters and FLOPs, denoted by CR-P and CR-F, respectively.
When the desired $\delta$ was not achieved for any compression ratio in the range the fields are left blank. The top values achieved for CR-P and CR-F are bolded.
}
\label{tab-supp:imagenet_retrain}
\begin{adjustbox}{width=1\textwidth}
\begin{tabular}{|c|c|c||ccc|ccc|ccc|ccc|ccc|}
\hline
\multirow{16}{*}{\rotatebox{90}{ImageNet}}
& \multirow{2}{*}{Model}
& \multirow{2}{*}{\shortstack{Prune \\ Method}}
& \multicolumn{3}{c|}{$\delta=0.0\%$}
& \multicolumn{3}{c|}{$\delta=0.5\%$}
& \multicolumn{3}{c|}{$\delta=1.0\%$}
& \multicolumn{3}{c|}{$\delta=2.0\%$}
& \multicolumn{3}{c|}{$\delta=3.0\%$} \\
& &
& Top1/5 Acc. & CR-P & CR-F
& Top1/5 Acc. & CR-P & CR-F
& Top1/5 Acc. & CR-P & CR-F
& Top1/5 Acc. & CR-P & CR-F
& Top1/5 Acc. & CR-P & CR-F \\ \cline{2-18}
& \multirow{7}{*}{\shortstack{ResNet18 \\ \\ Top1: 69.62 \\ Top5: 89.08}}
& ALDS
& +0.08/+0.18 & \textbf{59.43} & \textbf{33.08}
& -0.15/-0.03 & \textbf{66.73} & \textbf{44.14}
& -0.97/-0.49 & \textbf{72.73} & \textbf{52.81}
& -1.63/-0.76 & \textbf{77.62} & \textbf{60.46}
& -2.53/-1.44 & \textbf{81.75} & \textbf{67.62} \\
& & PCA
&   &   &  
&   &   &  
& -0.88/-0.43 & 9.97 & 12.02
& -1.84/-0.94 & 50.43 & 51.07
& -2.34/-1.23 & 59.38 & 60.08 \\
& & SVD-Energy
&   &   &  
&   &   &  
&   &   &  
& -1.91/-0.93 & 50.47 & 51.46
& -2.82/-1.53 & 59.42 & 60.24 \\
& & SVD
&   &   &  
&   &   &  
&   &   &  
& -1.53/-0.83 & 50.44 & 50.38
& -2.06/-1.03 & 59.36 & 59.33 \\
& & L-Rank
&   &   &  
&   &   &  
& -0.72/-0.26 & 10.01 & 32.40
& -0.72/-0.26 & 10.01 & 32.40
& -2.30/-1.26 & 26.25 & 58.59 \\
& & FT
& +0.17/+0.21 & 9.96 & 10.78
& +0.17/+0.21 & 9.96 & 10.78
& -0.66/-0.32 & 26.12 & 26.62
& -1.72/-0.81 & 39.58 & 37.89
& -2.82/-1.55 & 50.62 & 45.74 \\
& & PFP
& +0.25/+0.33 & 10.04 & 7.72
& -0.38/-0.15 & 26.35 & 19.14
& -0.38/-0.15 & 26.35 & 19.14
& -0.38/-0.15 & 26.35 & 19.14
& -2.80/-1.84 & 50.41 & 37.59 \\
\cline{2-18}
& \multirow{7}{*}{\shortstack{MobileNetV2 \\ \\ Top1: 71.85 \\ Top5: 90.33}}
& ALDS
&   &   &  
&   &   &  
&   &   &  
& -1.53/-0.73 & \textbf{32.97} & 11.01
& -1.53/-0.73 & 32.97 & 11.01 \\
& & PCA
&   &   &  
&   &   &  
& -0.87/-0.55 & \textbf{20.91} & 0.26
& -0.87/-0.55 & 20.91 & 0.26
& -0.87/-0.55 & 20.91 & 0.26 \\
& & SVD-Energy
&   &   &  
&   &   &  
&   &   &  
& -1.27/-0.57 & 20.02 & 8.57
& -2.50/-1.45 & 32.72 & 20.83 \\
& & SVD
&   &   &  
&   &   &  
&   &   &  
&   &   &  
&   &   &   \\
& & L-Rank
&   &   &  
&   &   &  
&   &   &  
&   &   &  
&   &   &   \\
& & FT
&   &   &  
&   &   &  
&   &   &  
& -1.73/-0.85 & 21.31 & \textbf{20.23}
& -2.68/-1.46 & 32.75 & \textbf{28.23} \\
& & PFP
&   &   &  
&   &   &  
& -0.97/-0.40 & 20.02 & \textbf{7.96}
& -1.44/-0.51 & \textbf{32.74} & 13.49
& -2.17/-0.85 & \textbf{43.32} & 19.21 \\
\hline
\end{tabular}

\end{adjustbox}
% \vspace{3px}
\endgroup
\end{table}

\begin{table}[t!]
\begingroup
\setlength{\tabcolsep}{3.05pt} % Default value: 6pt
\renewcommand{\arraystretch}{1.3} % Default value: 1
\centering
\caption{
The maximal compression ratio for which the drop in test accuracy is at most $\delta=1.0\%$ for ResNet18 (ImageNet) for various amounts of retraining (as indicated).
The table reports compression ratio in terms of parameters and FLOPs, denoted by CR-P and CR-F, respectively.
When the desired $\delta$ was not achieved for any compression ratio in the range the fields are left blank. The top values achieved for CR-P and CR-F are bolded.
}
\label{tab-supp:imagenet_retrainsweep}
\begin{adjustbox}{width=1\textwidth}
\begin{tabular}{|c|c|c||ccc|ccc|ccc|ccc|ccc|ccc|}
\hline
\multirow{9}{*}{\rotatebox{90}{ImageNet}}
& \multirow{2}{*}{Model}
& \multirow{2}{*}{\shortstack{Prune \\ Method}}
& \multicolumn{3}{c|}{$r=0\%\,e$}
& \multicolumn{3}{c|}{$r=5\%\,e$}
& \multicolumn{3}{c|}{$r=10\%\,e$}
& \multicolumn{3}{c|}{$r=25\%\,e$}
& \multicolumn{3}{c|}{$r=50\%\,e$}
& \multicolumn{3}{c|}{$r=100\%\,e$} \\
& &
& Top1/5 Acc. & CR-P & CR-F
& Top1/5 Acc. & CR-P & CR-F
& Top1/5 Acc. & CR-P & CR-F
& Top1/5 Acc. & CR-P & CR-F
& Top1/5 Acc. & CR-P & CR-F
& Top1/5 Acc. & CR-P & CR-F \\ \cline{2-21}
& \multirow{7}{*}{\shortstack{ResNet18 \\ \\ Top1: 69.62 \\ Top5: 89.08}}
& ALDS
& -0.54/-0.24 & \textbf{39.57} & \textbf{15.20}
& -0.48/-0.24 & \textbf{50.46} & \textbf{23.70}
& -0.72/-0.30 & \textbf{53.50} & \textbf{26.90}
& -0.64/-0.31 & \textbf{61.90} & \textbf{36.58}
& -0.40/-0.23 & \textbf{68.78} & \textbf{47.23}
& -0.73/-0.31 & \textbf{70.73} & \textbf{49.85} \\
& & PCA
&   &   &  
&   &   &  
& -inf/-inf & 3.33 & 3.99
& -0.80/-0.38 & 26.21 & 27.53
& -0.76/-0.51 & 39.53 & 40.45
& -inf/-inf & 6.65 & 8.14 \\
& & SVD-Energy
&   &   &  
& -0.28/-0.14 & 10.00 & 11.05
& -0.25/-0.12 & 10.00 & 11.05
& -0.55/-0.25 & 26.24 & 27.14
& -0.66/-0.33 & 39.56 & 40.48
& -inf/-inf & 3.33 & 3.66 \\
& & SVD
&   &   &  
& -0.32/-0.13 & 9.98 & 9.94
& -0.19/-0.07 & 9.98 & 9.94
& -0.71/-0.34 & 30.63 & 30.82
& -0.59/-0.32 & 39.53 & 39.51
& -inf/-inf & 3.33 & 3.31 \\
& & L-Rank
&   &   &  
&   &   &  
& -inf/-inf & 3.34 & 10.88
& -0.40/-0.23 & 10.01 & 32.40
& -0.16/+0.03 & 10.01 & 32.40
& -0.72/-0.26 & 10.01 & 32.40 \\
& & FT
&   &   &  
&   &   &  
& -inf/-inf & 3.36 & 3.75
& -0.21/-0.15 & 9.95 & 10.78
& -0.83/-0.46 & 26.29 & 26.57
& -0.66/-0.32 & 26.12 & 26.62 \\
& & PFP
&   &   &  
&   &   &  
& -inf/-inf & 3.34 & 2.37
& -0.14/-0.13 & 9.96 & 7.72
& -0.37/-0.31 & 20.76 & 15.14
& -0.38/-0.15 & 26.35 & 19.14 \\
\hline
\end{tabular}

\end{adjustbox}
\endgroup
\end{table}

\begin{table}[t!]
\begingroup
\setlength{\tabcolsep}{3.05pt} % Default value: 6pt
\renewcommand{\arraystretch}{1.2} % Default value: 1
\centering
\caption{
The maximal compression ratio for which the drop in test accuracy is at most some pre-specified $\delta$ on DeeplabV3-ResNet50 (Pascal VOC2012). The table reports compression ratio in terms of parameters and FLOPs, denoted by CR-P and CR-F, respectively.
When the desired $\delta$ was not achieved for any compression ratio in the range the fields are left blank. The top values achieved for CR-P and CR-F are bolded.
}
\begin{adjustbox}{width=1\textwidth}
\begin{tabular}{|c|c|c||ccc|ccc|ccc|ccc|ccc|}
\hline
\multirow{9}{*}{\rotatebox{90}{VOCSegmentation2012}}
& \multirow{2}{*}{Model}
& \multirow{2}{*}{\shortstack{Prune \\ Method}}
& \multicolumn{3}{c|}{$\delta=0.0\%$}
& \multicolumn{3}{c|}{$\delta=0.5\%$}
& \multicolumn{3}{c|}{$\delta=1.0\%$}
& \multicolumn{3}{c|}{$\delta=2.0\%$}
& \multicolumn{3}{c|}{$\delta=3.0\%$} \\
& &
& IoU/Top1 Acc. & CR-P & CR-F
& IoU/Top1 Acc. & CR-P & CR-F
& IoU/Top1 Acc. & CR-P & CR-F
& IoU/Top1 Acc. & CR-P & CR-F
& IoU/Top1 Acc. & CR-P & CR-F \\ \cline{2-18}
& \multirow{7}{*}{\shortstack{DeeplabV3-ResNet50 \\ \\ IoU: 68.16 \\ Top1: 94.25}}
& ALDS
& +0.14/-0.15 & \textbf{64.38} & \textbf{64.11}
& +0.14/-0.15 & \textbf{64.38} & \textbf{64.11}
& +0.14/-0.15 & \textbf{64.38} & \textbf{64.11}
& -1.22/-0.36 & \textbf{71.36} & \textbf{70.89}
& -2.76/-0.61 & \textbf{76.96} & \textbf{76.37} \\
& & PCA
&   &   &  
& -0.26/-0.02 & 31.59 & 31.63
& -0.88/-0.24 & 55.68 & 55.82
& -1.74/-0.39 & 64.33 & 64.54
& -2.54/-0.46 & 71.29 & 71.63 \\
& & SVD-Energy
&   &   &  
&   &   &  
&   &   &  
& -1.88/-0.47 & 31.61 & 32.27
& -2.78/-0.62 & 44.99 & 45.60 \\
& & SVD
& +0.01/-0.02 & 14.99 & 14.85
& -0.28/-0.18 & 31.64 & 31.51
& -0.89/-0.25 & 45.02 & 44.95
& -1.97/-0.50 & 64.42 & 64.42
& -1.97/-0.50 & 64.42 & 64.42 \\
& & L-Rank
&   &   &  
& -0.42/-0.09 & 44.99 & 45.02
& -0.42/-0.09 & 44.99 & 45.02
& -1.29/-0.33 & 55.74 & 56.01
& -2.50/-0.57 & 64.39 & 64.82 \\
& & FT
&   &   &  
&   &   &  
&   &   &  
&   &   &  
&   &   &   \\
& & PFP
& +0.01/-0.05 & 31.79 & 30.62
& -0.49/-0.21 & 45.17 & 43.93
& -0.84/-0.32 & 55.78 & 54.61
& -0.84/-0.32 & 55.78 & 54.61
& -2.43/-0.61 & 64.47 & 63.41 \\
\hline
\end{tabular}

\end{adjustbox}
% \vspace{3px}
\label{tab-supp:voc_retrain}
\endgroup
\vspace{-1ex}
\end{table}

\FloatBarrier
\subsection{Complete ImageNet Benchmark Results from Section~\ref{sec:experiments_iterative}}
\label{sec-supp:results_imagenet_more}
Results are provided in Table~\ref{tab-supp:imagenetbenchmarks}.

\begin{table}
    \centering
    \caption{AlexNet and ResNet18 Benchmarks on ImageNet. We report Top-1, Top-5 accuracy and percentage reduction in terms of parameters and FLOPs denoted by CR-P and CR-F, respectively. Best results with less than 0.5\% accuracy drop are bolded.}
    \label{tab-supp:imagenetbenchmarks}
    \renewcommand{\arraystretch}{1.2} % Default value: 1
    \begin{tabular}{|c|l|cccc|}
        \hline
        & Method & $\Delta$-Top1 & $\Delta$-Top5 & CR-P (\%) & CR-F (\%)  \\
        \hline\hline
        \multirow{19}{*}{{\rotatebox{90}{ \textbf{ResNet18}, Top1, 5: 69.64\%, 88.98\%}}}
        & ALDS (Ours) 
        & +0.41 & +0.37 &  66.70 &   42.70 \\
        & ALDS (Ours)
        & \textbf{-0.38} & \textbf{+0.04} & \textbf{75.00} & \textbf{64.50} \\
        & ALDS (Ours)
        & -0.90 & -0.25 &  78.50 &  71.50 \\
        & ALDS (Ours)
        & -1.37 & -0.56  & 80.60 & 76.30 \\
        \cline{2-6}
        & MUSCO~{\tiny\citep{gusak2019automated}}
        & \textbf{-0.37}   &  \textbf{-0.20} & N/A & 58.67\\
        & TRP1~{\tiny\citep{Xu2020}}
        & -4.18     &-2.5    & N/A & 44.70\\
        & TRP1+Nu~{\tiny\citep{Xu2020}}
        & -4.25    & -2.61   & N/A & 55.15\\
        & TRP2+Nu~{\tiny\citep{Xu2020}}
        & -4.3    & -2.37   & N/A & 68.55 \\    
        & PCA~{\tiny\citep{zhang2015accelerating}}
        & -6.54    &-4.54   & N/A & 29.07 \\
        & Expand~{\tiny\citep{jaderberg2014speeding}}
        & -6.84    & -5.26  & N/A & 50.00 \\
        \cline{2-6}
        & PFP~{\tiny\citep{liebenwein2020provable}}
        &  -2.26  & -1.07 & 43.80 & 29.30 \\
        & SoftNet~{\tiny\citep{he2018soft}}
        & -2.54  & -1.2 & N/A & 41.80 \\
        & Median~{\tiny\citep{he2019filter}}
        &  -1.23  & -0.5  & N/A & 41.80 \\
        & Slimming~{\tiny\citep{liu2017learning}}
        & -1.77    &   -1.19  & N/A & 28.05 \\
        & Low-cost~{\tiny\citep{dong2017more}}
        & -3.55    &  -2.2  & N/A & 34.64 \\
        & Gating~{\tiny\citep{hua2018channel}}
        & -1.52    &  -0.93 & N/A & 37.88 \\
        & FT~{\tiny\citep{he2017channel}}
        & -3.08    &  -1.75 & N/A & 41.86 \\
        & DCP~{\tiny\citep{zhuang2018discrimination}}
        & -2.19    & -1.28 & N/A & 47.08 \\
        & FBS~{\tiny\citep{gao2018dynamic}}
        & -2.44    &  -1.36 & N/A & 49.49\\
        \hline\hline
        \multirow{11}{*}{\rotatebox{90}{\small \textbf{AlexNet}, Top1, 5: 57.30\%, 80.20\%}}
        & ALDS (Ours)
        & +0.10 & +0.45 & 92.00 & 76.10 \\
        & ALDS (Ours) 
        & -0.21 & -0.36 & 93.0 & 77.9 \\
        & ALDS (Ours)
        & \textbf{-0.41} & \textbf{-0.54} & \textbf{93.50} & \textbf{81.4} \\
        \cline{2-6}
        & Tucker~{\tiny\citep{Kim2015}}
        & N/A  & -1.87 & N/A & 62.40 \\
        & Regularize~{\tiny\citep{tai2015convolutional}}
        & N/A  & -0.54 &  N/A & 74.35\\
        & Coordinate~{\tiny\citep{Wen2017}}
        & N/A  & -0.34 & N/A & 62.82 \\
        & Efficient~{\tiny\citep{kim2019efficient}}
        & -0.7 & -0.3  & N/A & 62.40 \\
        & L-Rank~{\tiny({\color{blue}Idelbayev et al.,~\citeyear{idelbayev2020low}})}
        & \textbf{-0.13} & \textbf{-0.13} & N/A & \textbf{66.77} \\
        \cline{2-6}
        & NISP~{\tiny\citep{yu2018nisp}}
        & -1.43    & N/A  & N/A & 67.94 \\ 
        & OICSR~{\tiny\citep{li2019oicsr}}
        & -0.47    &N/A   & N/A & 53.70  \\ 
        & Oracle~{\tiny\citep{ding2019approximated}}
        & -1.13    &-0.67 & N/A & 31.97 \\
        \hline
    \end{tabular}
\end{table}

\subsection{Ablation Study}
\label{sec-supp:results_ablation}

In order to gain a better understanding of the various aspects of our method we consider an ablation study where we selectively turn off various features of ALDS. Specifically, we compare the full version of ALDS to the following variants: 

\begin{enumerate}[itemsep=5pt, leftmargin=25pt, labelwidth=10pt, topsep=5pt, partopsep=0pt]
    \item[\large 1.] 
    \textsc{ALDS-Error} solves for the optimal ranks (Line~\ref{lin:opt_ranks} of Algorithm~\ref{alg:budget_allocation}) for a desired set of values for $k^1,\ldots, k^L$. We test $k^\ell=3,\,\forall \ell \in [L]$. This variant tests the benefits of varying the number of subspaces compared to fixing them to a desired value.
    \item[\large 2.]
    \textsc{SVD-Error} corresponds to ALDS-Error with $k^\ell=1,\,\forall \ell \in [L]$. This variants tests the benefits of having multiple subspaces in the first places in the context error-based allocation of the per-layer compression ratio.
    \item[\large 3.] 
    \textsc{ALDS-Simple} picks the ranks in each layer for a desired set of values of $k^1,\ldots, k^L$ such that the per-layer compression ratio is constant. We test $k^\ell=3,\,\forall \ell \in [L]$, and $k^\ell=5,\,\forall \ell \in [L]$. This variant tests the benefits of allocating the per-layer compression ratio according to the layer error compared to a simple constant heuristic.
    \item[\large 4.] 
    \textsc{Messi} proceeds like ALDS-Simple but replaces the subspace clustering with projective clustering~\citep{maalouf2021deep}. We test $k^\ell=3,\,\forall \ell \in [L]$. This variant tests the disadvantages of having a simple subspace clustering technique (channel slicing) compared to using a more sophisticated technique.
\end{enumerate}

We note that ALDS-Simple with $k^\ell=1,\,\forall \ell \in [L]$ corresponds to the SVD comparison method from the previous sections.

We study the variations on a ResNet20 trained on CIFAR10 in two settings: compression only and one-shot compress+retrain. The results are presented in Figures~\ref{fig-supp:ablation}. We highlight that the complete variant of our algorithm (ALDS) consistently outperforms the weaker variants providing empirical evidence on the effectiveness of each of the core components of ALDS.

We note that varying the number of subspaces for each layer in order to optimally assign a value of $k^\ell$ in each layer is crucial in improving our performance. This is apparent from the comparison between ALDS, ALDS-Error, and SVD-Error: having a fixed value for k yields sub-optimal results. 

Picking an appropriate notion of cost (maximum relative error) is furthermore preferred over simple heuristics such a constant per-layer compression ratio. Specifically, the main difference between ALDS-Error and ALDS-Simple is the way how the ranks are determined for a given set of $k$'s: ALDS-Error optimizes for the error-based cost function while ALDS-Simple relies on a simple constant per-layer compression ratio heuristic. In practice, ALDS-Error outperforms ALDS-Simple across all tested scenarios.

Finally, we test the disadvantages of using a simple subspace clustering method. To this end, we compare ALDS-Simple and Messi for fixed values of $k$. While in some scenarios, particularly without retraining, Messi provides modest improvements over ALDS-Simple, the improvement is negligible for most settings. Moreover, note that Messi requires an expensive approximation algorithm as explained in Section~\ref{sec-supp:method_clustering}. This would in turn prevent us from incorporating Messi into the full ALDS framework in a computationally efficient manner. However, as apparent from the ablation study we exhibit the most performance gains for features related to global considerations instead of local, per-layer improvements. 
In addition, we should also note that Messi does not emit a structured reparameterization thus requires specialized software or hardware to obtain speed-ups. Consequently, we may conclude that channel slicing is the appropriate clustering technique in our context.

\begin{figure}
\centering
\begin{minipage}[t]{0.35\textwidth}
    \includegraphics[width=\textwidth]{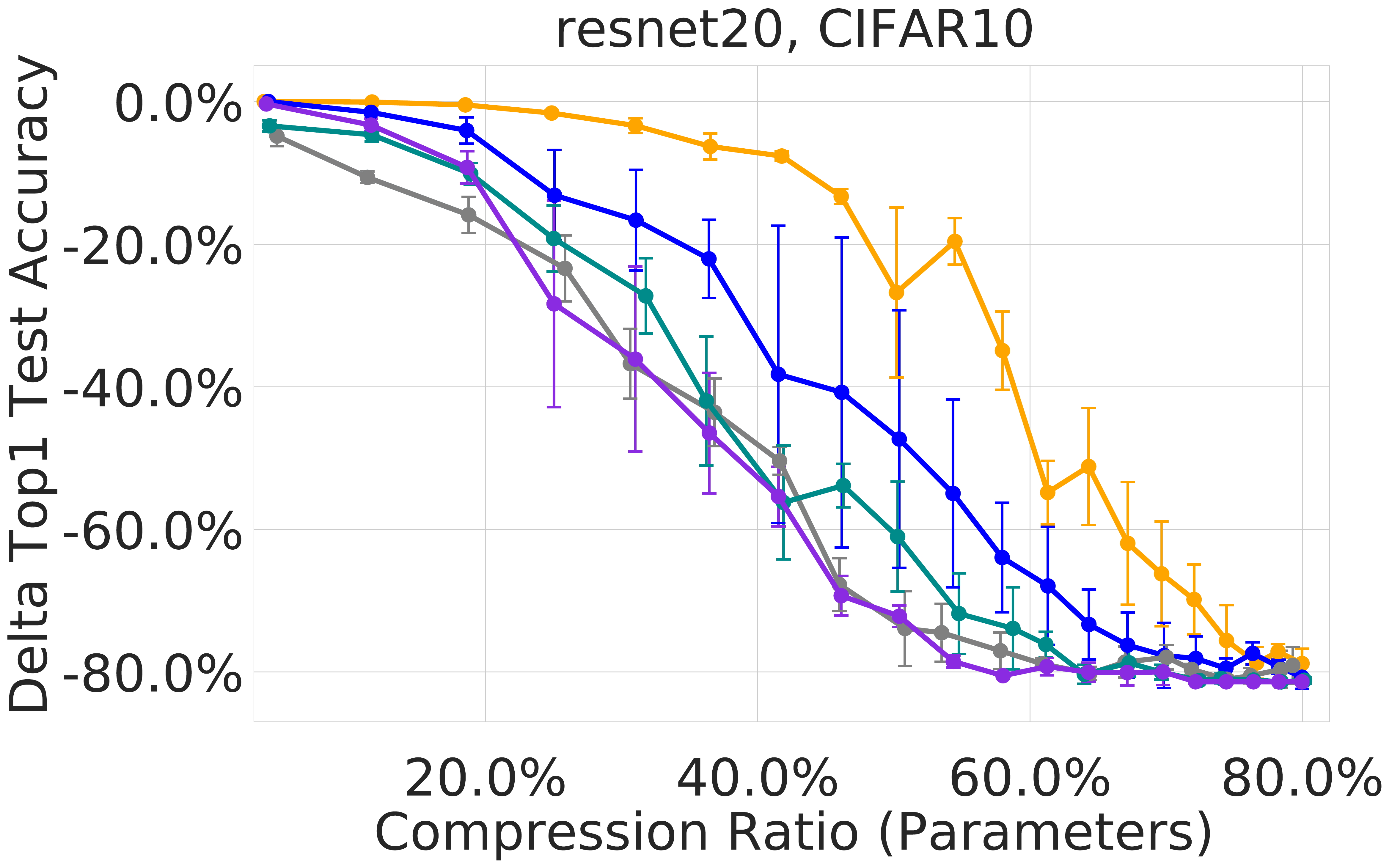}
    \subcaption{no retraining ($r=0$)}
\end{minipage}%
\begin{minipage}[t]{0.35\textwidth}
    \includegraphics[width=\textwidth]{fig/cifar_retrainablation/resnet20_CIFAR10_e182_re182_retrain_int20_CIFAR10_acc_param.pdf}
    \subcaption{one-shot ($r=e$)}
\end{minipage}%
\hspace{3ex}%
\begin{minipage}[t]{0.14\textwidth}%
    \vspace{-15ex}%
    \includegraphics[width=\textwidth]{fig/legend/ablation.pdf}
\end{minipage}%
\caption{
The difference in test accuracy (``Delta Top1 Test Accuracy'') for various target compression ratios, ALDS-based/ALDS-related methods, and networks on CIFAR10.}
\label{fig-supp:ablation}
\end{figure}

\newpage
\subsection{Extensions of ALDS}
\label{sec-supp:results_extension}
We test and compare ALDS with ALDS+ (see Section~\ref{sec-supp:method_extension}) to investigate the performance gains we can obtain from generalizing our local step to search over multiple decomposition schemes. We run one-shot compress-only experiments on ResNet20 (CIFAR10) and ResNet18 (ImageNet). 

The results are shown in Figure~\ref{fig-supp:extension}. We find that ALDS+ can significantly increase the performance-size trade-off compared to our standards ALDS method. This is expected since by generalizing the local step of ALDS we are increasing the search space of possible decomposition solution. Using our ALDS framework we can efficiently and automatically search over the increased solution space. We envision that our observations will invigorate future research into the possibility of not only choosing the optimal per-layer compression ratio but also the optimal compression scheme.

\begin{figure}[H]
\centering
\begin{minipage}[t]{0.35\textwidth}
    \includegraphics[width=\textwidth]{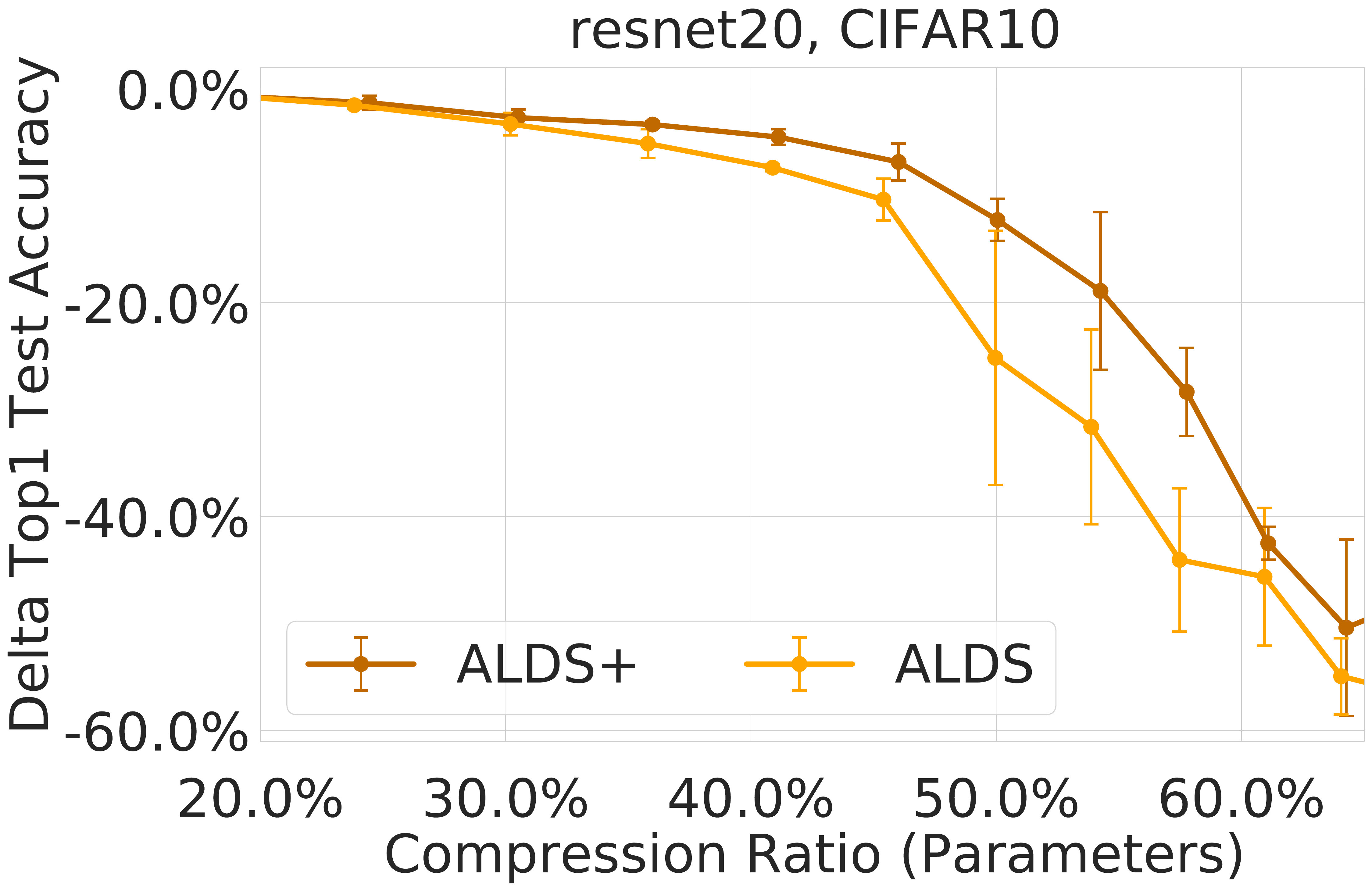}
    \subcaption{ResNet20 (CIFAR10)}
\end{minipage}%
\hspace{3ex}%
\begin{minipage}[t]{0.375\textwidth}
    \includegraphics[width=\textwidth]{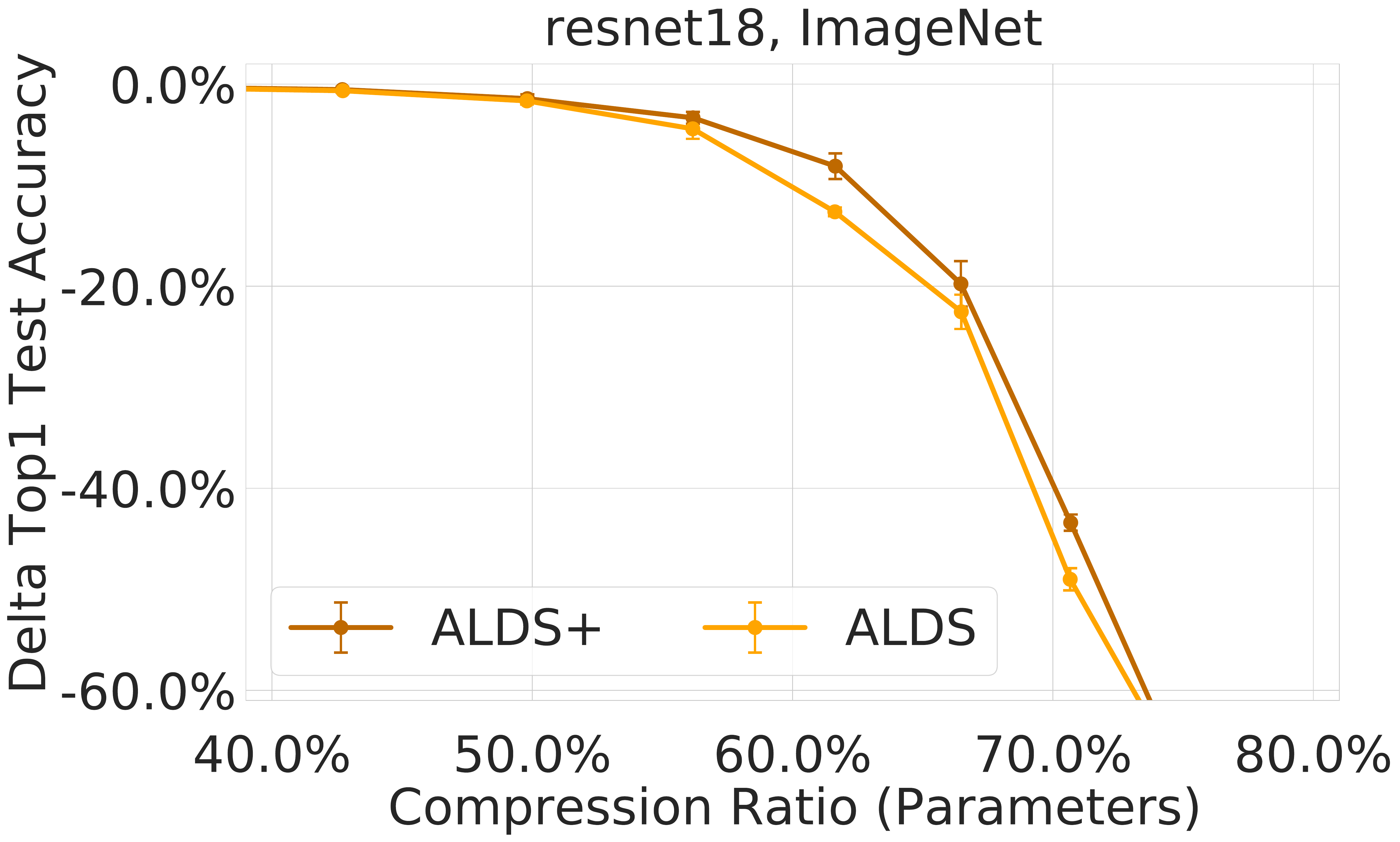}
    \subcaption{ResNet18 (ImageNet)}
\end{minipage}%
\caption{
The difference in test accuracy (``Delta Top1 Test Accuracy'') for various target compression ratios, ALDS-based/ALDS-related methods, and networks on CIFAR10. The networks were compressed once and not retrained afterwards.}
\label{fig-supp:extension}
\end{figure}

% \bibliography{misc/references}

\end{document}